\documentclass[10pt,twocolumn,letterpaper]{article}

\usepackage{cvpr}              %

\usepackage[dvipsnames]{xcolor}

\usepackage{multirow}

\newcommand{\I}{\mathcal{I}}

\newcommand{\F}{{\mathcal{F}}}

\newcommand{\newtext}[1]{#1}

\usepackage{amsmath, amssymb, amsthm}
\usepackage{mleftright}

\newcommand{\net}{{\mathcal{N}}}

\newcommand{\R}{{\mathbb{R}}}

\newcommand{\pare}[1]{\left(#1\right)}

\newcommand{\sqbrac}[1]{{\left[{#1}\right]}}

\newcommand{\bmat}[1]{\begin{bmatrix}#1\end{bmatrix}}

\DeclareMathOperator*{\E}{\mathbb{E}}

\newtheorem{theorem}{Theorem}[section]
\newtheorem{lemma}[theorem]{Lemma}

\theoremstyle{definition}

\theoremstyle{remark}

\usepackage[accsupp]{axessibility} 

\definecolor{cvprblue}{rgb}{0.21,0.49,0.74}
\usepackage[pagebackref,breaklinks,colorlinks,citecolor=cvprblue]{hyperref}

\def\paperID{13356} %
\def\confName{CVPR}
\def\confYear{2024}

\title{CodedEvents: Optimal Point-Spread-Function Engineering for\\ 3D-Tracking with Event Cameras}

\author{
Sachin Shah \and Matthew A. Chan \and Haoming Cai \and Jingxi Chen \and Sakshum Kulshrestha \and Chahat Deep Singh \and Yiannis Aloimonos \and Christopher A. Metzler \\
\and
{\small University of Maryland, College Park}\\
\vspace{-15pt}
{\tt\small shah2022@umd.edu}
}

\begin{document}

\twocolumn[{%
\renewcommand\twocolumn[1][]{#1}%
\maketitle
\begin{center}
    \centering
    \captionsetup{type=figure}
    \includegraphics[width=0.9\textwidth]{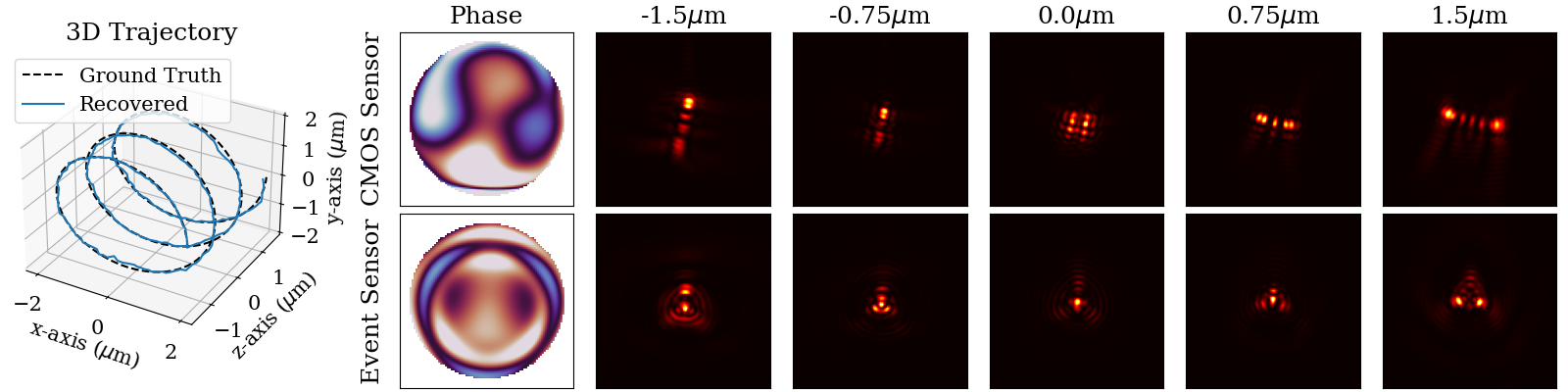}
    \captionof{figure}{ \textbf{CodedEvent Tracking.} Left: example recovered trajectory using designed optics for an event camera. Right: top row, optimal phase mask design and PSFs for a CMOS sensor, bottom row, our optimal phase mask design and PSFs for an event sensor. }
\end{center}%
}]

\begin{abstract}
Point-spread-function (PSF) engineering is a well-established computational imaging technique that uses phase masks and other optical elements to embed extra information (e.g.,~depth) into the images captured by conventional CMOS image sensors.
To date, however, PSF-engineering has not been applied to neuromorphic event cameras; a powerful new image sensing technology that responds to changes in the log-intensity of light. 

This paper establishes theoretical limits (Cram\'er Rao bounds) on 3D point localization and tracking with PSF-engineered event cameras.  
Using these bounds, we first demonstrate that existing Fisher phase masks are already near-optimal for localizing static flashing point sources (e.g., blinking fluorescent molecules). 
We then demonstrate that existing designs are sub-optimal for tracking moving point sources and proceed to use our theory to design optimal phase masks and binary amplitude masks for this task. 
To overcome the non-convexity of the design problem, we leverage novel implicit neural representation based parameterizations of the phase and amplitude masks. 
We demonstrate the efficacy of our designs through extensive simulations. %
We also validate our method with a simple prototype.

\end{abstract}

\section{Introduction}
\label{sec:intro}

Single-molecule localization microscopy (SMLM) is a vital tool for resolving nano-scale structures with applications in analysis of protein clusters~\cite{Maynard:2021}, cell dynamics~\cite{Verdier:2022}, and electromagnetic effects~\cite{Kinkhabwala:2009}. Traditional SMLM experiments are limited by the slow capturing process of frame-based CMOS sensors, preventing use in capturing high-speed, dynamic interactions. Recently, \cite{Cabriel:2023:SMLM} showed event cameras are key to enabling high-speed 2D SMLM.

In contrast to traditional CMOS cameras, event cameras are an emerging class of bio-inspired neuromorphic sensors that operate with a high temporal resolution on the order of $\mu$s. These sensors are comprised of an asynchronous pixel array, where each pixel records an event when the log intensity change exceeds a set threshold. In addition to having kilohertz time resolution, these sensors are low-power, resistant to constant background noise, and can operate over a high dynamic range~\cite{Gallego:2020}. Already, these sensors have proven useful in a range of applications including object tracking~\cite{Angelopoulos:2021:eyetracking, Tinch:2022:stars}, gesture recognition~\cite{Lee:2014:gesture, Amir:2017:gesture}, and robotics~\cite{Iaboni:2021:Robots, Gomez:2022:Robots}. 

Just as PSF-engineering allows one to extract additional information using conventional CMOS sensors~\cite{Shechtman:2014}, 
we believe that event-camera-specific PSF engineering will be the key to enabling high-speed 3D SMLM with event cameras. 
Unfortunately, existing PSF design theory is not equipped for the event space. 
In this work, we bridge this gap by developing Cram\'er Rao Bounds on 3D position estimation for event camera measurements. Leveraging these bounds, we subsequently develop a novel implicit neural representation for optical elements to design components with improved 3D particle localization capabilities.

Specifically, our principal contributions are as follows:
\begin{itemize}
    \item We derive the Fisher Information and Cram\'er Rao Bounds for event camera measurements parameterized by 3D spatial positions.
    \item We develop novel implicit neural representations for learning both amplitude and phase masks.
    \item We identify new phase and amplitude designs for optimally encoding 3D information with event cameras.
    \item We demonstrate in simulation that our designs outperform existing methods at 3D particle tracking.%
\end{itemize}

\section{Related Work}
\label{sec:related}

\subsection{Coded Optics}
Specialized lenses have been shown to encode additional depth information in CMOS image frames. A `coded aperture' can produce depth-dependent blurs that enable one to extract depth by looking at the per-pixel defocus pattern~\cite{levin2007image}. Future works extend the `depth from defocus' idea by leveraging information theory to design an optimal lens~\cite{Shechtman:2014, Jusuf:2022:OPSF}. More recently, researchers have proposed optimizing optical parameters in conjunction with a neural network reconstruction algorithm in an `end-to-end' fashion. This joint-optimization problem is difficult to optimize due to local minima. Many works have discussed mask parameterizations to stabilize optimization: Zernike basis~\cite{Wu:2019:PhaseCam, Chang:2019:DeepOptics3D} and rotationally symmetric~\cite{Ikoma:2021:DepthFromDefocus}. However, direct pixel-wise methods should be preferred due to their expressiveness~\cite{Liu:2022}. Dynamic pixel-wise masks have been proposed as a training stabilization mechanism~\cite{Shah:2023:TiDy}. Specialized optics have been explored for other applications such as super resolution~\cite{Sitzmann:2018:EDOF}, high-dynamic-range imaging~\cite{Metzler:2020:HDR}, hyper-spectral sensing~\cite{Li:2022:HS}, and privacy-preservation~\cite{Hinojosa:2021:privacy}. To our knowledge, PSF engineering specifically for event-based sensors has been relatively unexplored.

\subsection{Microscopy Tracking}
Originally, single-particle localization was limited to 2D dimensions, where only the $x,y$ coordinates of an emitter are recovered~\cite{Small:2014:SPL}. Similar to works on depth from defocus, the depth of an emitter can be recovered from 2D measurements by considering a microscope's PSF. A standard microscope typically has a PSF resembling the circular Airy pattern; however, because it spreads out quickly its depth resolving range is limited. A few engineered PSFs---such as the double-helix PSF \cite{Pavani:2009:DHPSF}---have since been proposed  to improve the imaging range. In particular, Shechtman \etal finds the optimally informative PSF (dubbed the Fisher PSF) for a CMOS sensor to localize the 3D position of a single emitter~\cite{Shechtman:2014}. A few other techniques for resolving the 3D location of particles have been proposed such as light-field-microscopy~\cite{Llavador:2016:LFM} and lensless imaging~\cite{LindaLiu:2020:Lensless}.

Unfortunately, these techniques are limited by the sub-kilohertz readout of conventional CMOS sensors. This hinders their use in imaging fast, dynamic processes such as blood flow \cite{Bouchard:2009:BloodFlow} and voltage signals~\cite{Abdelfattah:2023:VoltageSignal}. A few ultrafast imaging methods have also been proposed \cite{Gao:2014:ultrafast, Wu:2020:twophoton, Ma:2021:Flourescence, Xiao:2023:twophoton} but require high-power illumination which can be phototoxic to certain organic samples. Recently, event cameras have been proposed as an alternative to CMOS sensors for 2D SMLM \cite{Cabriel:2023:SMLM}. Another work proposes extending light-field-microscopy to event cameras to resolve 3D position but requires complex optical setups and sacrifices spatial resolution \cite{Guo:2023:EventLFM}. By designing optics to encode depth information into event streams, we can enable high-speed 3D SMLM.

\subsection{Depth Estimation}

Extracting 2D information from images tends to be a significantly easier task than extracting depth, hence, monocular depth estimation is often the bottleneck in 3D tracking performance. 
Structured light projectors~\cite{Geng:2011:StructuredLight} or time-of-flight sensors~\cite{Foix:2011:tof} use active illumination to extract depth information. Given these methods' reliance on an internal light source, performance can degrade in adverse lighting conditions.
If we allow multiple views, stereo~\cite{Hartley:2004:MV} or structure from motion~\cite{Tomasi:1992} can triangulate 3D position. These methods are sensitive to occlusion and texture-less scenes and require multiple calibrated cameras.
Many neural network approaches with all-in-focus CMOS images as input have been proposed~\cite{Spencer:2020:Depth, You:2021:Depth, Ranftl:2022:MiDaS, Yin:2023:Metric3D}. Recently, event-based depth estimation has made significant progress with neural networks~\cite{Zhu:2019:EventDepth, Hidalgo:2020:EventDepth,Mostafavi:2021:EventRGBStereo, Nam:2022:EventStereo, Shi:2023:EVEN}. Spiking neural networks have been proposed for spiking cameras, which similar to event cameras, offer asynchronous readout of pixels~\cite{Zhang:2022:Spike}.

\section{Theory}\label{sec:theory}

\subsection{Event Camera Simulation}
\newtext{Let $(x(t), y(t), z(t))$ be the location of a point light source at time $t$. We focus on tracking points around some focal plane $z$, with $z(t)=z+\Delta z(t)$ and $z\gg|z(t)|$. In this context, a pin-hole camera would capture,
\begin{align}
    I_t(u, v) &= \delta\pare{ u - f\frac{x(t)}{z+\Delta z(t)}, v -f\frac{ y(t)}{z+\Delta z(t)} }\\
    &\approx \delta\pare{ u - \frac{f}{z}x(t), v -\frac{f}{z}y(t) }
\end{align}
where $\delta$ is the Dirac Delta function. Because $f$ and $z$ are constant, we will consider $x(t)$ and $y(t)$ pre-scaled for notation sake. }
In practice, a camera captures a blurry image depending on the point-spread-function (PSF) it induces.
A PSF $h$ can be modeled with Fourier optics theory as a function of 3D-position $x,y,z$, amplitude modulation $A$ caused by blocking light, and phase modulation $\phi^M$ caused by phase mask height variation~\cite{Goodman:2017}. 
\begin{align}
    h=\left|\F\left[ A \exp\left( i \phi^{DF}(x,y,z) + i \phi^{M}  \right) \right]\right|^2
\end{align}
where $\phi^{DF}(x,y,z)$ is the defocus aberration due to the distance from the camera.
Then, a point light source at location $(x(t), y(t), z(t))$ captured by a regular camera is
\begin{align}
    I^b_t(u, v) &= [h_{z(t)} * I_t](u, v)\\
                &= h(x(t), y(t) ; z(t)).
\end{align}
Note that because this PSF depends on depth, it can be used to encode depth information into $I^b$. Event cameras trigger events with respect to the log of photocurrent $L = \log(I^b)$~\cite{Gallego:2020} \newtext{where a pixel's photocurrent is linearly related to the wave intensity at that  pixel}. Specifically, an event is triggered when the absolute difference between the current intensity at $t+\tau$ and the reference intensity from $t$, $\Delta L(u, v) = L_{t+\tau}(u, v) - L_{t}(u, v)$, is greater than some threshold $T$.
\begin{equation}\label{eq:Otpiece}
    O_t(u, v) = \begin{cases} 
      +1 & \Delta L(u, v) > T \\
      -1 & \Delta L(u, v) < -T \\
      \textnormal{none} & \textnormal{otherwise} 
   \end{cases}
\end{equation}
In isolation, each event contains little information; however, a sequence of events can be highly informative~\cite{Lagorce:2017:HOTS, Sironi:2018:HATS, Alzugaray:2018:FSAE}. 
Notably the inceptive event time-surfaces representation suggests the trailing events that occur after the first event correspond to the log-intensity change~\cite{Baldwin:2019:LogDiff}. Therefore, by binning events over time, one can approximately recover the change in log intensity $\Delta L$. Visuallly, we show the accumulated event frame approaches $\Delta L$ as the number of intermediate frames accumulated increases in~\autoref{fig:log-diff-sweep}. \newtext{We prove this approximation is at most off by $1$ for an idealized event camera in Section~S4 of the supplement.} Therefore, our event measurement \eqref{eq:Otpiece} can be simplified as,
\begin{align}
    O_t &= \log\pare{I_t^b} - \log\pare{I^b_{t-\tau}}. \label{eq:optimistic-evt}
\end{align}

\begin{figure*}
    \centering
    \includegraphics[width=0.95\linewidth]{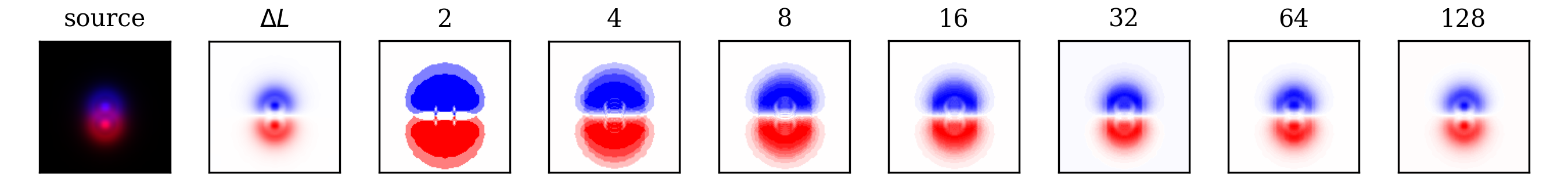}
    \caption{\textbf{ Binning events approximates the log difference as the number of accumulated frames increases. } Consider a point source moving from the blue location to the red location at depth plane $1\mu$m over a fixed time interval in the first image. The second image illustrates the direct access to the difference in \eqref{eq:optimistic-evt}, while the subsequent images demonstrate the effect of accumulating $N$ event frames across the time interval. Observe how large $N$ nearly recovers $\Delta L$, demonstrating the validity of the approximation.}
    \label{fig:log-diff-sweep}
    \vspace{-10pt}
\end{figure*}

\subsection{Information}
In the field of statistical information theory, the Fisher Information (FI) reports the amount of information gained about the parameters of a distribution, given a measurement. As such, we can use FI to express the effectiveness of PSFs at encoding depth information. The multi-parameter FI is represented as an $N\times N$ matrix where the $i,j$ entry is defined as the variance of the score:
\begin{align}
    \I(\theta)_{i,j} = \E\sqbrac{ \pare{ \frac{\partial}{\partial\theta_i} \log{f(X ; \theta)} }\pare{ \frac{\partial}{\partial\theta_j} \log{f(X ; \theta)} } \mid \theta }
\end{align}
where $\theta$ is the set of parameters, $\theta_i$ is the $i$th parameter, and $f(X ; \theta)$ is a probability density function for the distribution observation $X$ is drawn from. 

For traditional CMOS sensors, FI has been used to compare coded apertures and phase masks for a wide range of tasks such as depth estimation~\cite{Ober:2004}, hyper-spectral imaging~\cite{Baek:2021:HS}, \newtext{and detecting linear structures~\cite{Ghanekar:2022:ps2f}}.
Those works have shown that the intrinsic photon shot noise in $I^b$ can be modeled as a Poisson random variable with mean $\lambda=h(x,y,z)$. We derive the FI matrix for an event sensor.

\noindent\textbf{Flashing light.} As a warm-up, consider the SMLM technique for event cameras presented in \cite{Cabriel:2023:SMLM}, which assumes a blinking labeling model similar to STORM (stochastic optical reconstruction microscopy) \cite{Rust:2006:STORM}, PALM (photoactivated localization microscopy) \cite{Betzig:2006:PALM} and DNA-PAINT (DNA point accumulation for imaging in nano-scale topography) \cite{Sharonov:2006:DNA-PAINT}. With this idealized model of an event camera, $\log I^b_{t-\tau} = 0$, so \eqref{eq:optimistic-evt} reduces to
\begin{align}
    O_t = \log I^b_t.
\end{align}
By applying $e^x$ to the measurement, we can indirectly measure $I^b_t$. Moreover, by applying standard results for FI of a Poisson distribution \cite{Snyder:1991, Kay:1993}, we can write the FI matrix for an event camera capturing a blinking particle as:
\begin{equation}
    \I(\theta)_{i,j} = \sum_{n}^N \frac{1}{h(n) + \beta} \left( \frac{\partial h(n) }{\partial\theta_i}
  \right) \left( \frac{\partial h_z(n) }{\partial\theta_j}
  \right) \label{eq:fisher-general}
\end{equation}
where $N$ is the number of pixels, $h(n)$ is the PSF intensity at pixel $n$, $\beta$ is background noise, and $\theta = \{x, y, z\}$ corresponds to the 3D location of a point source. Notice that this is the same result as in \cite{Ober:2004}, suggesting that --- in the context of blinking particles --- the Fisher mask found in~\cite{Shechtman:2014} for a traditional CMOS camera is also optimal for an event-based sensor. 

\noindent\textbf{Generalization.} We now derive the positional information content for any event measurement. Rewriting \eqref{eq:optimistic-evt} with logarithmic rules, we obtain,
\begin{align}
    O_t = \log{ \frac{ I^b_t }{ I^b_{t-\tau} } }.
\end{align}
The inner expression is drawn from the ratio of Poisson random variables with means $\lambda_{t}$ and $\lambda_{t-\tau}$. This can be approximated as a single Normal distribution~\cite{Griffin:1992}:
\begin{align}
    \frac{I^b_{t}}{I^b_{t-\tau}}  \sim \net\pare{
        \frac{\lambda_t}{\lambda_{t-\tau}},
        \frac{\lambda_t}{\lambda_{t-\tau}^2} + \frac{\lambda_t^2}{\lambda_{t-\tau}^3}
    }.
\end{align}
Similar to the flashing light example, we can exponentiate the measurement to recover this ratio.
Using the symbolic mathematics solver SymPy \cite{SymPy}, we evaluate the expectation in~\eqref{eq:fisher-general} with 
$\theta=\{ x_t, y_t, z_t,  x_{t-\tau}, y_{t-\tau}, z_{t-\tau} \}$
and $f(X;\theta)$ as the PDF of the normal distribution, yielding
\begin{align}
    \I(\theta) = \sum_n^N \frac{\mathcal{D}^T\mathcal{D}}{2\pare{\mu+\nu}^2 }  \odot
    \bmat{  a&a&a&b&b&b\\
            a&a&a&b&b&b\\ 
            a&a&a&b&b&b\\ 
            b&b&b&c&c&c\\ 
            b&b&b&c&c&c\\ 
            b&b&b&c&c&c
    }
\end{align}
where
\begin{align}
    \mu &= \lambda_{t-\tau} = h(x_{t-\tau},y_{t-\tau},z_{t-\tau})+\beta \\
    \nu &= \lambda_{t}      = h(x_{t},y_{t},z_{t}) +\beta\\
    \mu_i &= \frac{\partial}{\partial\theta_i} \mu \\
    \nu_i &= \frac{\partial}{\partial\theta_i} \nu \\
    \mathcal{D} &= \bmat{ \mu_x/\mu & \mu_y/\mu & \mu_z/\mu & \nu_x/\nu & \nu_y/\nu & \nu_z/\nu } \\
    a &= 2 \mu^{2} \nu + 4 \mu^{2} + 2 \mu \nu^{2} + 12 \mu \nu + 9 \nu^{2} \\
    b &=  -\pare{2\mu^{2} \nu + 2 \mu^{2} + 2 \mu \nu^{2} + 7 \mu \nu + 6 \nu^{2}} \\
    c &= 2 \mu^{2} \nu + \mu^{2} + 2 \mu \nu^{2} + 4 \mu \nu + 4 \nu^{2} 
\end{align}

\section{Method}

\begin{figure*}[t]
\centering
\subfloat[Optical component optimization.]{
  \includegraphics[clip,width=0.85\linewidth]{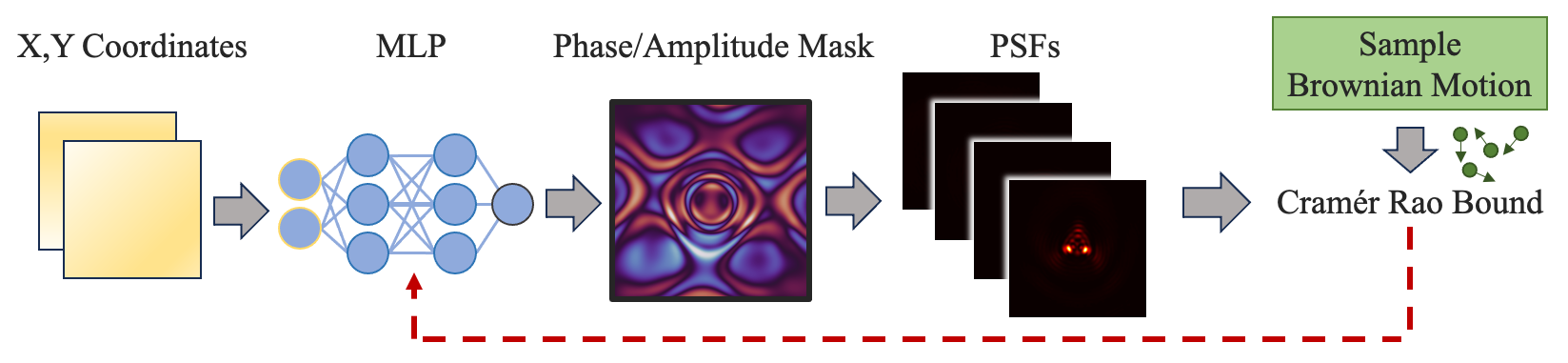}%
}

\subfloat[Coded event 3D tracking]{
  \includegraphics[clip,width=0.9\linewidth]{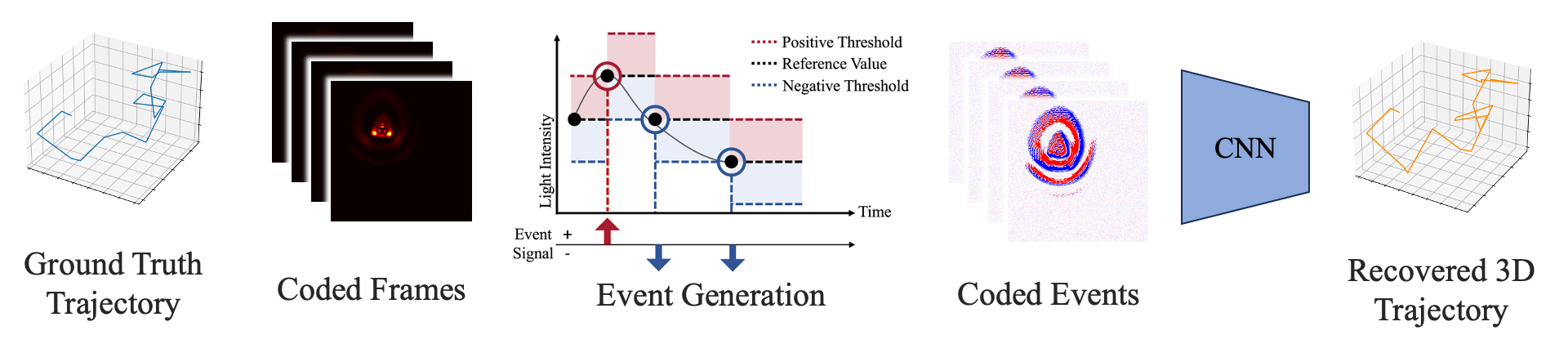}
}
\caption{\textbf{System overview.} (a) An MLP produces a phase or amplitude mask based on a grid of $x,y$ coordinates. The weights are updated through back-propagation of the CRB computed with Brownian Motion. (b) In simulation, coded events are generated by first rendering high-frame-rate coded CMOS frames and converting them to event frames. These measurements are passed to a 3D-tracking algorithm.}
\end{figure*}

\subsection{Objective Function}
Similar to existing work on 3D tracking for CMOS sensors, we can leverage the FI matrix to optimize optical parameters that efficiently encode depth information~\cite{Shechtman:2014, Wu:2019:PhaseCam}. Specifically, we compute the Cram\'er Rao Bound (CRB), which provides a fundamental bound on how accurately parameters can be estimated given a measurement. If $T(X)$ is the unbiased estimator for parameters $\theta$, then the CRB is
\begin{align}
    CRB_i \equiv \sqbrac{\I(\theta)^{-1}}_i \leq \text{cov}_\theta\pare{ T(X) }_i .
\end{align}
Then, the objective function we wish to minimize is
\begin{align}
    \mathcal{L}_{CRB} = \sum_{z\in Z} \sum_{i\in\theta} \sqrt{ \sqbrac{\I(\theta)^{-1}}_{i,i} }
\end{align}
where $Z$ is a set of depth planes.

\subsection{Optical Parameter Representation}
PSF manipulation is typically achieved through designed optical elements such as phase and amplitude masks. In general, phase masks are preferred over binary amplitude masks for their photon efficiency and continuous parametric representation, allowing for optimization via standard gradient descent methods. 
Inspired by~\cite{NeuWS}, we demonstrate that implicit neural representations can model phase masks in such a way that results in more stable optimization and better-optimized mask designs. We use an architecture similar to the sinusoidal representation network (SIREN) presented in \cite{Sitzmann:2019:SIREN} to predict the phase delay caused by the mask at each location $(u, v)$. Input data in $\R^2$ is processed by a four-layer multi-layer perceptron (MLP) with hidden feature size $128$, and $\sin$ activation. We refer to this method as \textit{Neural Phase Mask (NPM)}.

Phase masks offer many degrees of freedom and excellent light throughput, but can be relatively expensive to manufacture and are only effective for some frequencies.
Meanwhile binary amplitude masks are cheap to manufacture (such as with consumer-grade 3D printers) and can operate across all frequencies (including x-ray), but offer fewer degrees of freedom. %

Historically, methods for designing optimal binary apertures have been fundamentally limited due to the lack of optimization techniques for discrete binary parameters. As a result, prior works~\cite{levin2007image, nayar2009coded, nayar2009what} walk over a restricted search space, leaving ample room for improvement. To solve this issue, we propose a novel implicit neural representation for binary amplitude masks. We use an MLP to predict the percent of photons blocked at each mask location $(u,v)$. The input in $\R^2$ is processed by a four-layer MLP with hidden feature size $128$ and SoftPlus~\cite{Nair:2010:SoftPlus} activation. The output to the network is passed through a sigmoid. %
We refer to this method as \textit{Neural Amplitude Mask (NAM)}.

\section{Experimental Details}

PSFs are simulated for a microscope imaging system with NA$=1.4$, index of refraction $n=1.518$, wavelength $\lambda=550$nm, magnification $M=111.11$, $4$f lens focal length $f=150$mm, pixel pitch of $49.58\mu$m, and resolution of $256\times 256$. Each phase and amplitude mask is optimized using $\mathcal{L}_{CRB}$ for 10,000 epochs. Because particle motion influences FI, we leverage Monte Carlo sampling while training to maximize information content for all motion directions. For each epoch, we compute the total CRB for $3$ random orthogonal motions across $11$ depth planes.
We use the Adam~\cite{KingBa15} optimizer with parameters $\beta_1=0.99$, $\beta_2=0.999$, and a learning rate of $10^{-3}$. Training and testing were conducted on NVIDIA RTX A5000 GPUs.

To validate our design's ability to track point sources, we train a Convolutional Neural Network (CNN) to map binned event frames to 3D locations. Events are accumulated over $16$ refresh cycles to produce an accumulated event frame. These $256\times 256$ single-channel images are processed by a CNN with 5 convolutional blocks and a linear output head. Each block is followed by batch normalization, ELU activation~\cite{Clevert:2016:ELU}, and max pooling. The output is a normalized length $3$ vector representing the position of the particle at a given time step. The CNN is trained on $3$ Brownian motion trajectories. Each trajectory is sampled at 16,000 time steps. A `coded' CMOS video frame is simulated by blurring a $300$nm emitter with the optical component's PSF for the location \newtext{and adding Gaussian noise (to simulate other noise sources such as thermal)}. Next, we generate a `coded-event-stream' from the high-speed video using standard event camera simulator methods by tracking the per-pixel reference signal~\cite{Hu:2021:v2e}. Finally, we bin every $16$ frames to produce a $1000$-frame `coded-event-video'. The particle location at the end of the $16$-frame bin is considered the ground truth position. We supplement this training with $2000$ random starting positions and corresponding motion vectors. Each motion is scaled to have magnitude drawn from $\net\pare{100\text{nm}, 20\text{nm}}$. For each position-motion pair, we generate a $16$ frame `coded' CMOS video to accumulate into a `coded-event-frame'. The CNN is trained for 100 epochs with the Adam optimizer.
\newtext{
We also manufacture a lab prototype to the demonstrate practical benefits of coded apertures for event cameras (see Section S1 in the supplementary materials for details).
}

\section{Results}
Because designed optics for event cameras is an emerging field, we compare our optimized phase and amplitude mask designs to components designed for traditional CMOS sensors: open aperture/Fresnel lens, Fisher phase mask~\cite{Shechtman:2014} and Levin \etal's amplitude mask~\cite{levin2007image} (\autoref{fig:other-optics}). 

\begin{figure}[t]
    \centering
    \includegraphics[width=\linewidth]{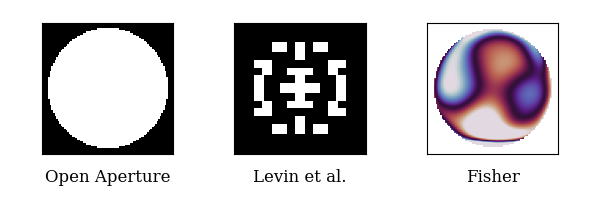}
    \caption{\textbf{Visualization of non-event camera-specific optical components.} Each component is placed in the same plane as a $150$mm focal length lens.}
    \label{fig:other-optics}
\end{figure}

\subsection{Cram\'er Rao Bound}
We simulate Brownian motion by sampling $1000$ unit direction vectors and independently scaling them by a magnitude drawn from $\net\pare{100\text{nm}, 20\text{nm}}$. The speed is relative to the event camera refresh rate, with a $1000$ accumulated-event-frame per second system, this motion simulates a range of biological processes such as molecular diffusion~\cite{Yang:2015:Diffusion}.
We then evaluate the average CRB over the $1000$ motions at $30$ depth planes spaced evenly on a $3\mu$m range around the focal plane. For all $6$ position parameters, we plot the CRB trend with respect to depth (\autoref{fig:crb-others}). Observe that each optical system performs worse as a point source moves away from the focal plane as the defocus change decreases. Although an open-aperture lens is slightly better around the focal plane, its bound increases at a higher rate than the other designs. We also report the average CRB over all parameters and depth slices to demonstrate our neural-based phase mask is best overall (\autoref{tab:crb-all}).

\begin{table}[t]
\begin{center}
\begin{tabular}{ccc}
                           & Component     & CRB (nm) $\downarrow$        \\ \hline
                           & Open Aperture &  80.8  \\ \hline
\multirow{2}{*}{Amplitude} & Levin \etal   &  263.3  \\
                           & NAM (Ours)      &  50.5  \\ \hline
\multirow{2}{*}{Phase}     & Fisher        &  36.3 \\
                           & NPM (Ours)      &  \textbf{33.1}     
\end{tabular}
\end{center}
\caption{ \textbf{Average CRB for each optical component} across a $3\mu$m depth range for all $6$ position parameters. Phase masks outperform amplitude masks due to higher light efficiency, and our neural-designed phase mask is best. }
\label{tab:crb-all}
\end{table}

\begin{figure}
    \centering
    \includegraphics[width=\linewidth]{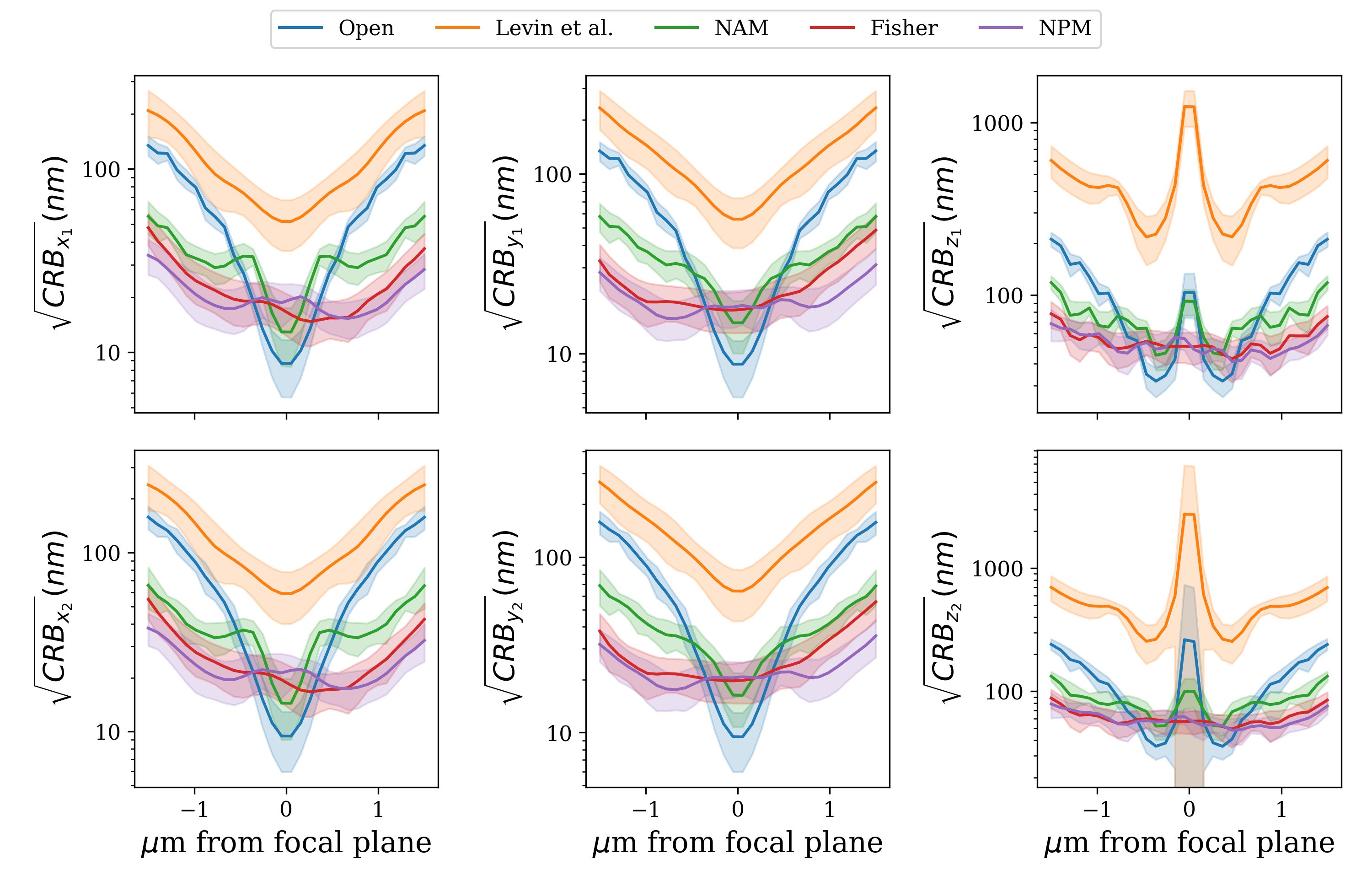}
    \caption{\textbf{3D localization CRB with respect to depth}. First row: particle's $x,y,z$ position at time $t-\tau$. Second row: particle's $x,y,z$ position at time $t$. Observe the bound increases as the source drifts from the focal plane.}
    \label{fig:crb-others}
\end{figure}

\subsection{3D Tracking}
We validate our theoretical results in simulation by tracking a 3D moving emitter across a $8\mu\text{m}\times 8\mu\text{m}\times 4\mu\text{m}$ volume. After training a CNN to decode 3D position from coded event frames, we evaluate our network tracking performance on $5$ sequences of Brownian motion, each consisting of $1000$ binned frames.
\autoref{tab:tracking} shows our event camera-specific optical designs minimize 3D tracking error more than conventional designs. Additionally, our method is substantially better at depth plane recovery. Qualitative results in \autoref{fig:tracking} demonstrate that 3D positions recovered using our designs more tightly fit ground-truth trajectories.

\begin{table}
\begin{center}
\begin{tabular}{cccc}
                           &               &RMSE (nm) $\downarrow$ & $L_1$ (nm) $\downarrow$ \\
                           & Component     & 3D   & $z$ \\ \hline
                           & Open Aperture &  617   &         936   \\ \hline
\multirow{2}{*}{Amplitude} & Levin \etal   &  764   &         1036   \\
                           & NAM (Ours)      &  66.0  &         49.2   \\ \hline
\multirow{2}{*}{Phase}     & Fisher        &  52.6  &         44.2   \\
                           & NPM (Ours)      &  \textbf{51.2} &  \textbf{39.2}                
\end{tabular}
\end{center}
\caption{ \textbf{Tracking accuracy comparison.} We present quantitative results on 3D trajectory recovery for known optical designs. Our event CRB loss function found the best-performing design. Although only slightly improved in overall 3D tracking, our design noticeably improves depth recovery. }
\label{tab:tracking}
\end{table}

\begin{figure}
    \centering
    \includegraphics[width=\linewidth]{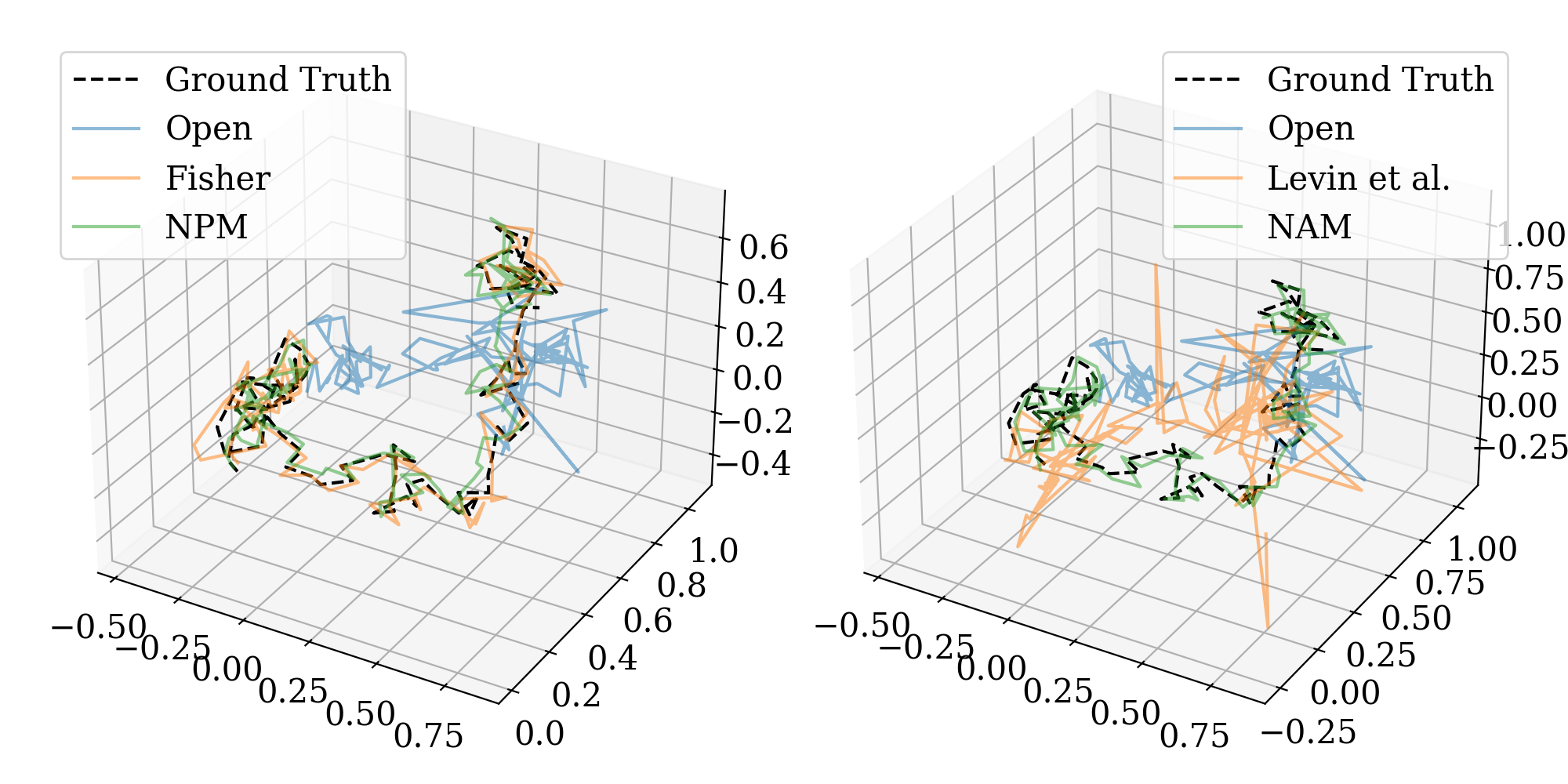}
    \caption{\textbf{Recovered 3D position over Brownian motion sequence with coded event frames.} Left: phase mask methods, right: amplitude mask methods. Observe trajectories reconstructed from phase mask-coded events more closely align with ground-truth positions. Units in microns. }
    \label{fig:tracking}
\end{figure}

\section{Ablation Studies}
\subsection{Optical Representations}

Additionally, we compare 3D tracking results using two different amplitude mask representations: pixel-wise and neural amplitude mask (\autoref{fig:learned-amplitude}) and three different phase mask representations: pixel-wise, Zernike basis, and neural phase mask (\autoref{fig:learned-phase}). As shown in \autoref{tab:crb-parameterization}, our implicit neural representation-based methods achieve a lower average error bound than alternative representations, despite being two times smaller than pixel-wise representations with respect to the number of parameters. As expected, phase mask results generally outperform the amplitude mask results (\autoref{fig:crb-parameterization}). However, our novel neural binary aperture makes optimizing amplitude masks more tractable. We observe that pixel-wise representations not only yield difficult-to-manufacture apertures but also suboptimal performance.
In terms of 3D tracking, the implicit neural representations produce a smaller error on average (\autoref{tab:tracking-abla}) and more accurately match sampled 3D trajectories (\autoref{fig:tracking-abla}).

\begin{figure}
    \centering
    \includegraphics[width=\linewidth]{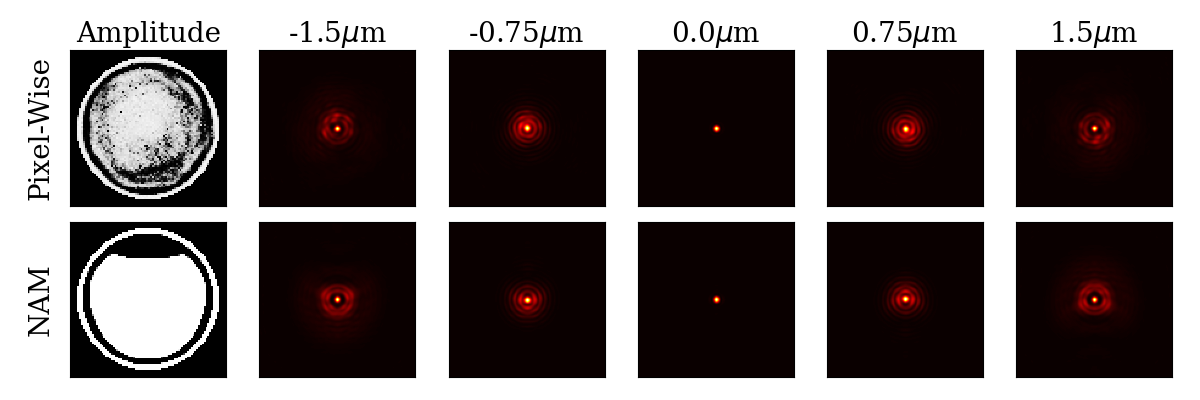}
    \caption{\textbf{Designed amplitude masks and corresponding PSFs.} Top: pixel-wise representation. Bottom: implicit neural representation.}
    \label{fig:learned-amplitude}
\end{figure}

\begin{figure}
    \centering
    \includegraphics[width=\linewidth]{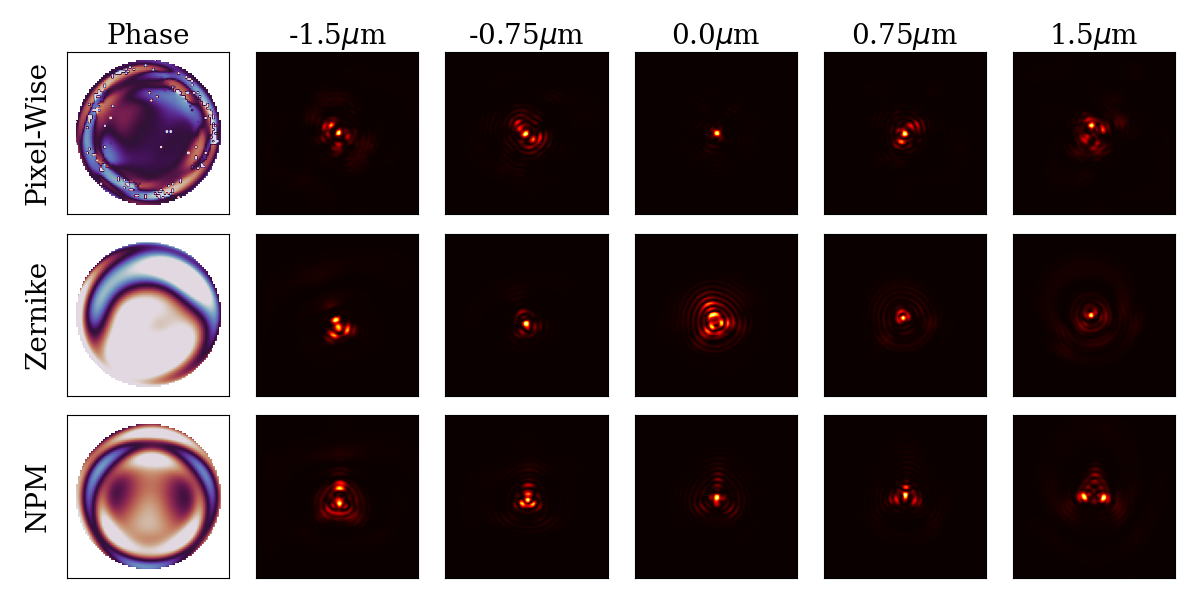}
    \caption{\textbf{Designed phase masks and corresponding PSFs.} Top: pixel-wise representation. Middle: first 55 Zernike coefficients representation. Bottom: implicit neural representation.}
    \label{fig:learned-phase}
\end{figure}

\begin{table}
\begin{center}
\begin{tabular}{ccc}
                           & Representation     & CRB (nm) $\downarrow$ \\ \hline
\multirow{2}{*}{Amplitude} & Pixel-Wise    &  65.5       \\
                           & NAM           &  50.5       \\ \hline
\multirow{3}{*}{Phase}     & Pixel-Wise    &  34.2       \\
                           & Zernike       &  34.8       \\
                           & NPM           &  \textbf{33.1}     
\end{tabular}
\end{center}
\caption{ \textbf{Average CRB of different optimized representations across a $3\mu$m depth range}. Notice the neural representations outperform their pixel-wise counterparts. }
\label{tab:crb-parameterization}
\end{table}

\begin{figure}
    \centering
    \includegraphics[width=\linewidth]{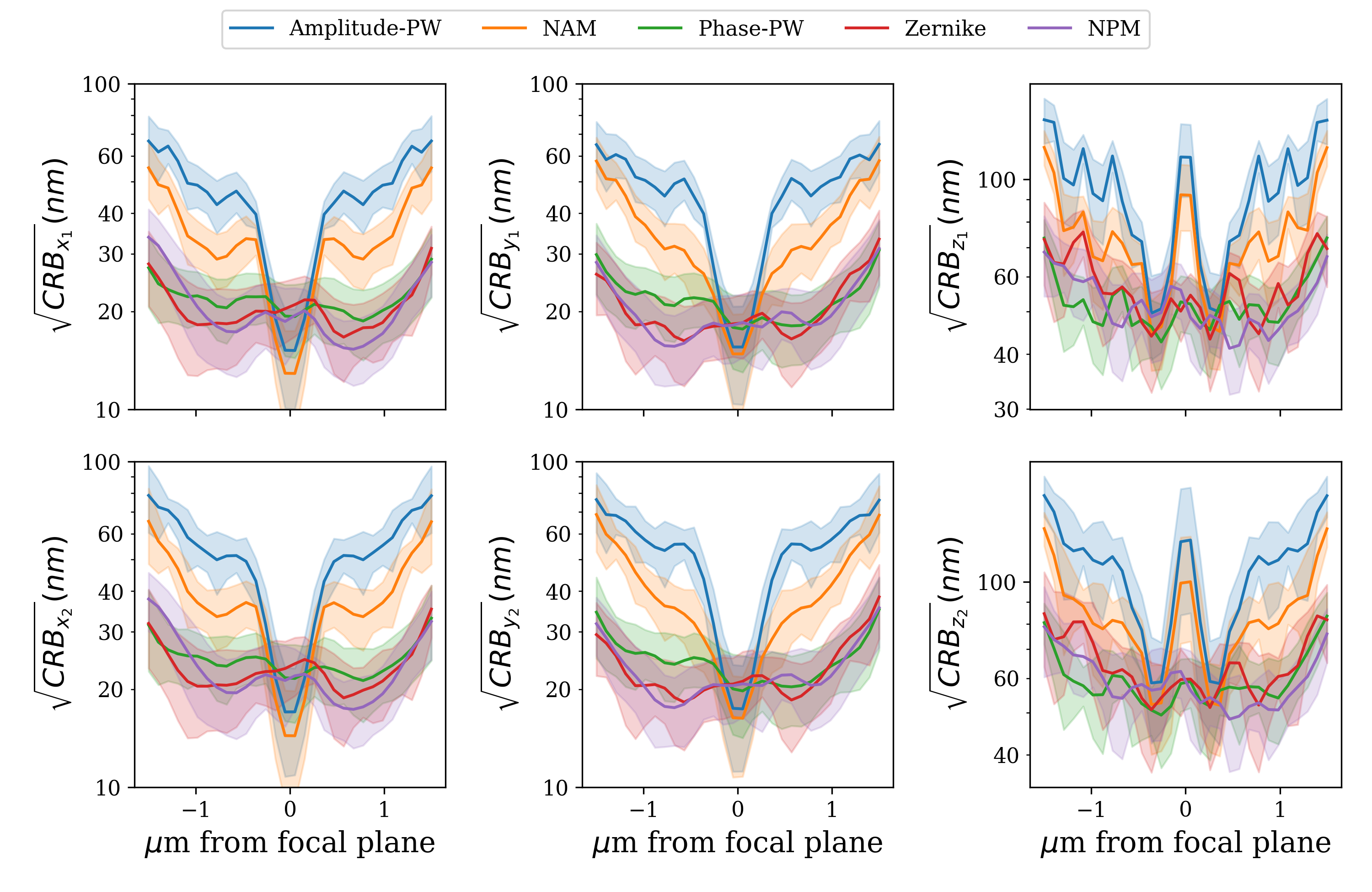}
    \caption{\textbf{Effect of optical parameterization on 3D localization CRB}. First row: particle's $x,y,z$ position at time $t-\tau$. Second row: particle's $x,y,z$ position at time $t$. Our implicit neural representations are particularly advantageous for amplitude masks.}\label{fig:crb-parameterization}
\end{figure}

\begin{table}
\begin{center}
\begin{tabular}{cccc}
                           &               &RMSE (nm) $\downarrow$ & $L_1$ (nm) $\downarrow$ \\
                           & Component   & 3D     & $z$ \\ \hline
\multirow{2}{*}{Amplitude} & Pixel-Wise  &  120   &         103   \\
                           & NAM         &  66.0  &         49.2   \\ \hline
\multirow{2}{*}{Phase}     & Pixel-Wise  &  56.5  &         45.9   \\
                           & Zernike     &  51.3  &         50.2   \\
                           & Our NPM     &  \textbf{51.2} &  \textbf{39.2}                
\end{tabular}
\end{center}
\caption{ \textbf{Effect of optimized mask parameterization on tracking accuracy.} Average distance between ground-truth Brownian motion and the recovered 3D position is minimized with our neural-based designs.}
\label{tab:tracking-abla}
\end{table}

\begin{figure}
    \centering
    \includegraphics[width=\linewidth]{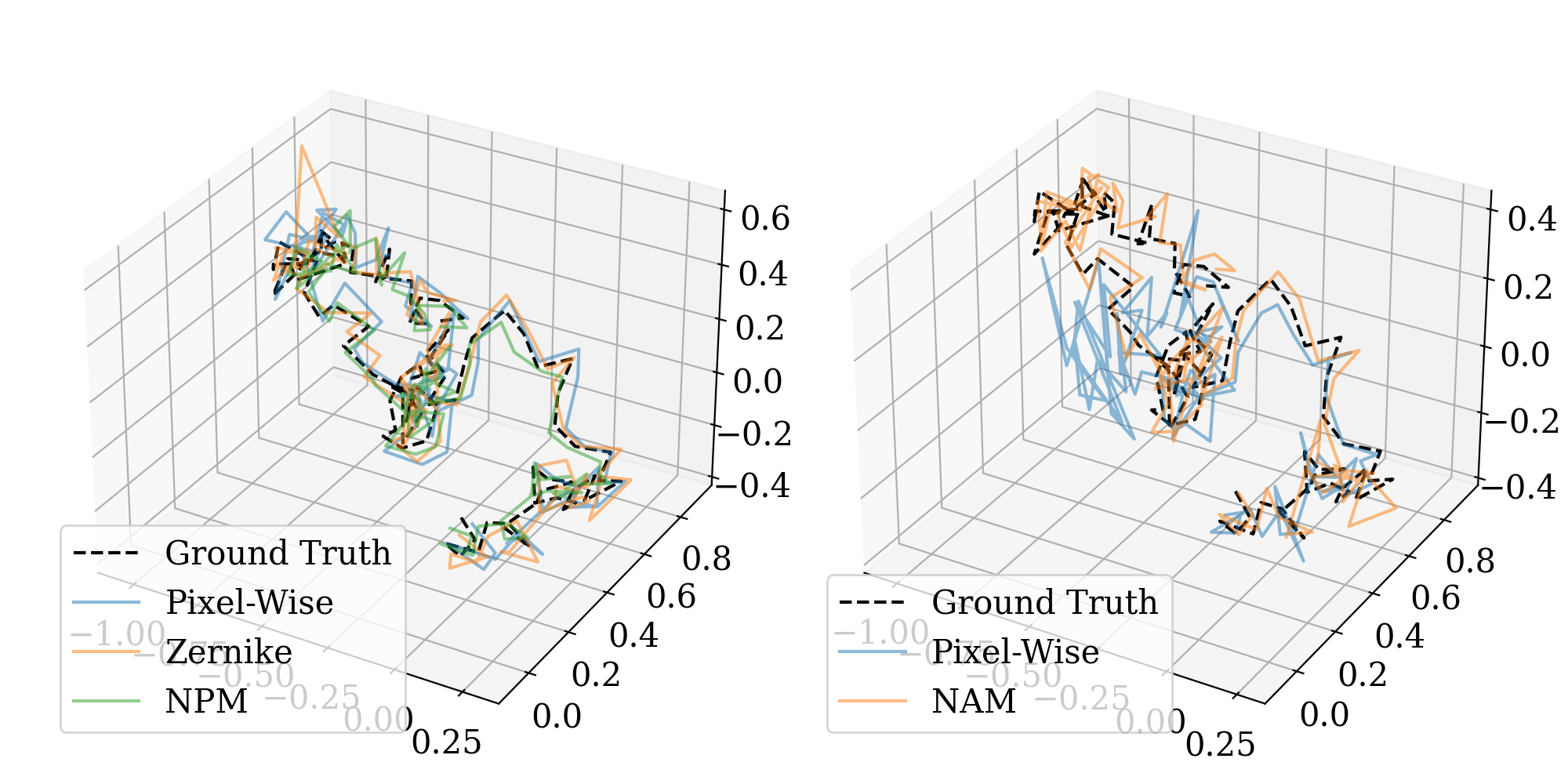}
    \caption{\textbf{Effect of optical representation on 3D trajectory recovery.} %
    Left: phase mask methods, right: amplitude mask methods. Observe that neural representations produce tighter reconstructions. Units in microns. }
    \label{fig:tracking-abla}
\end{figure}

\subsection{Tracking Limits}

In this section, we explore the limits of 3D tracking with variable external factors. For each experiment, we compute the average CRB over $30$ depth slices and $6$ parameters for $3$ orthogonal unit directions ($x$, $y$, and $z$). First, as the number of available photons increases, the lower bound on 3D position estimation monotonically decreases~(\autoref{fig:misc-photons}). More available photons equate to a higher signal-to-noise ratio. Additionally, this result helps explain why phase masks outperform amplitude masks.
Second, we show extremely slow-moving particles (less than nanometers per refresh rate) experience a significantly higher CRB~(\autoref{fig:misc-speed}). Minimal movement indicates smaller intensity changes and thus an event camera would trigger fewer events. On the other side, as a particle moves faster, the number of events will decrease as there is a non-zero delay between when an event camera can trigger sequential events. Our learned phase mask is more robust to speed changes than an open aperture and our learned amplitude mask.
Third, when the percentage of photons due to background noise increases, the bound on error also increases (\autoref{fig:misc2}). We design our masks with $1\%$ of captured photons attributable to the background, but the learned designs are more resistant to degraded conditions than an open aperture.

\newtext{
We also explore the effect of modifying the accumulation period in Section S2 and how the optimal design changes with respect to speed in Section S3 of the supplement.
}

\begin{figure}
    \centering
    \includegraphics[width=\linewidth]{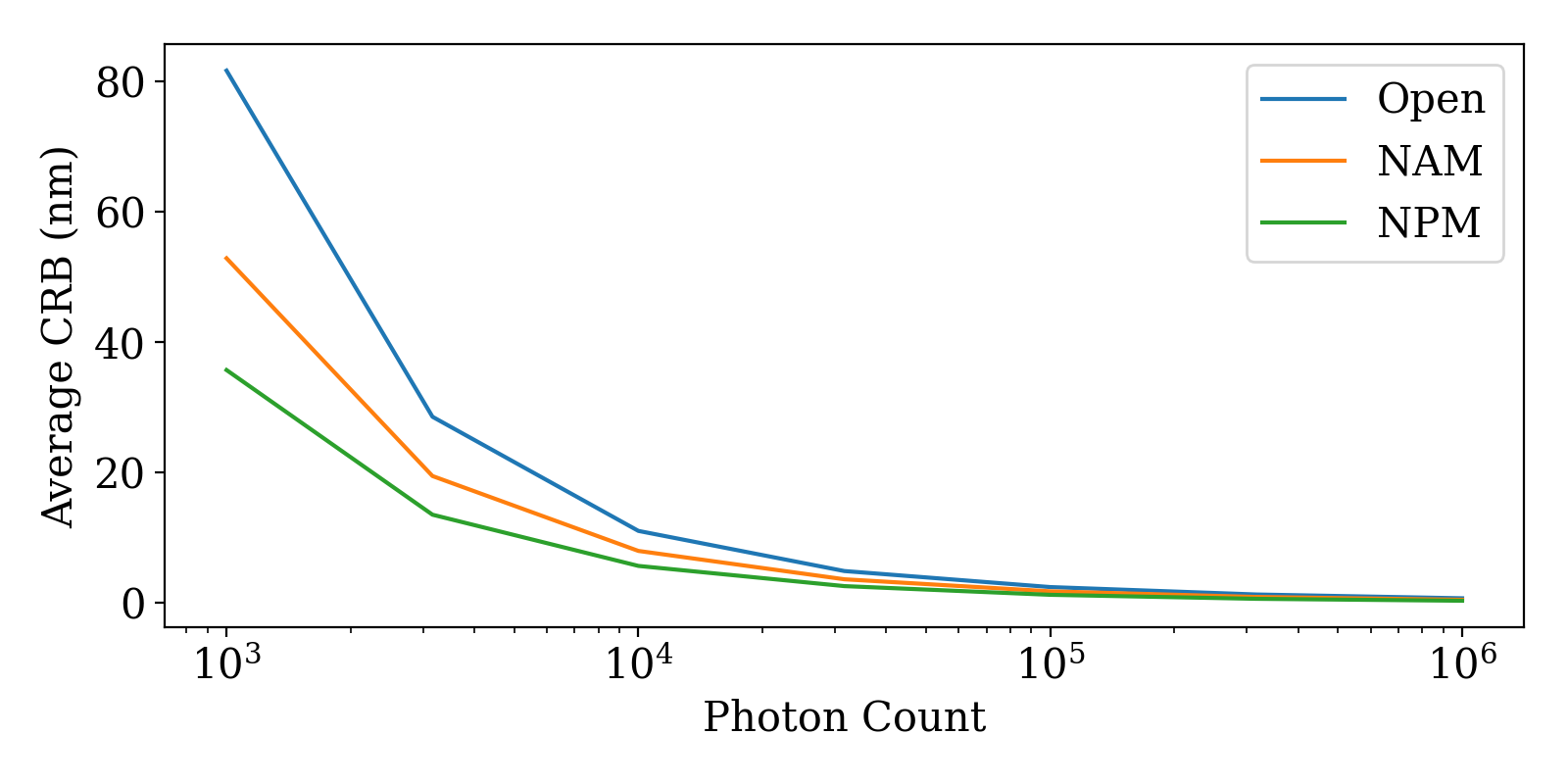}
    \caption{ \textbf{Flux effect on CRB.} With more available photons, the signal-to-noise ratio increases, so the 3D information content is more reliable, and the bound on 3D tracking error decreases. }
    \label{fig:misc-photons}
\end{figure}

\begin{figure}
    \centering
    \includegraphics[width=\linewidth]{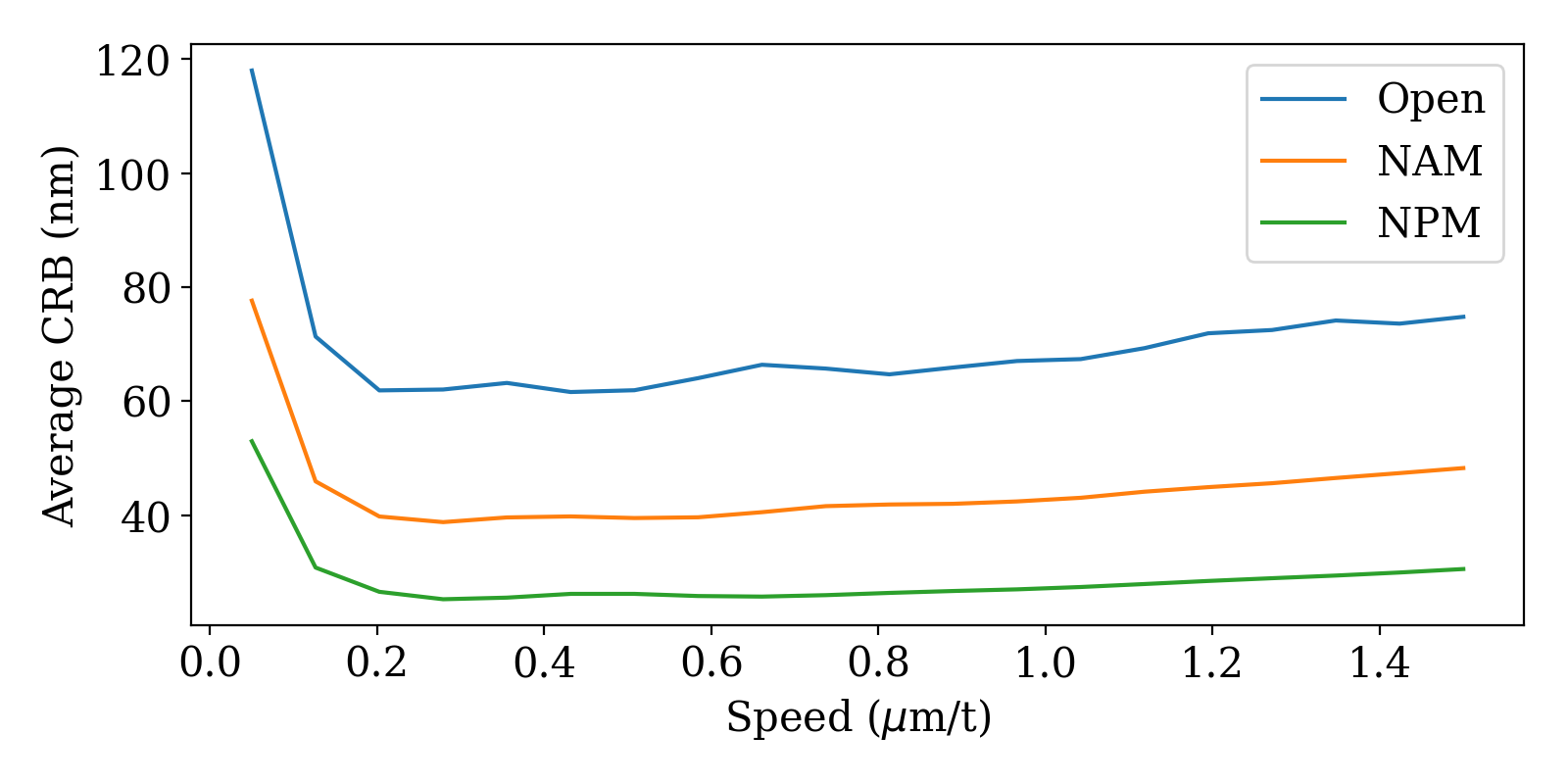}
    \caption{ \textbf{Speed effect on CRB.} Too-slow moving particles trigger fewer events yielding a worse CRB. Similarly, as a particle moves faster the delay between triggers leads to fewer events. }
    \label{fig:misc-speed}
\end{figure}

\begin{figure}
    \centering
    \includegraphics[width=\linewidth]{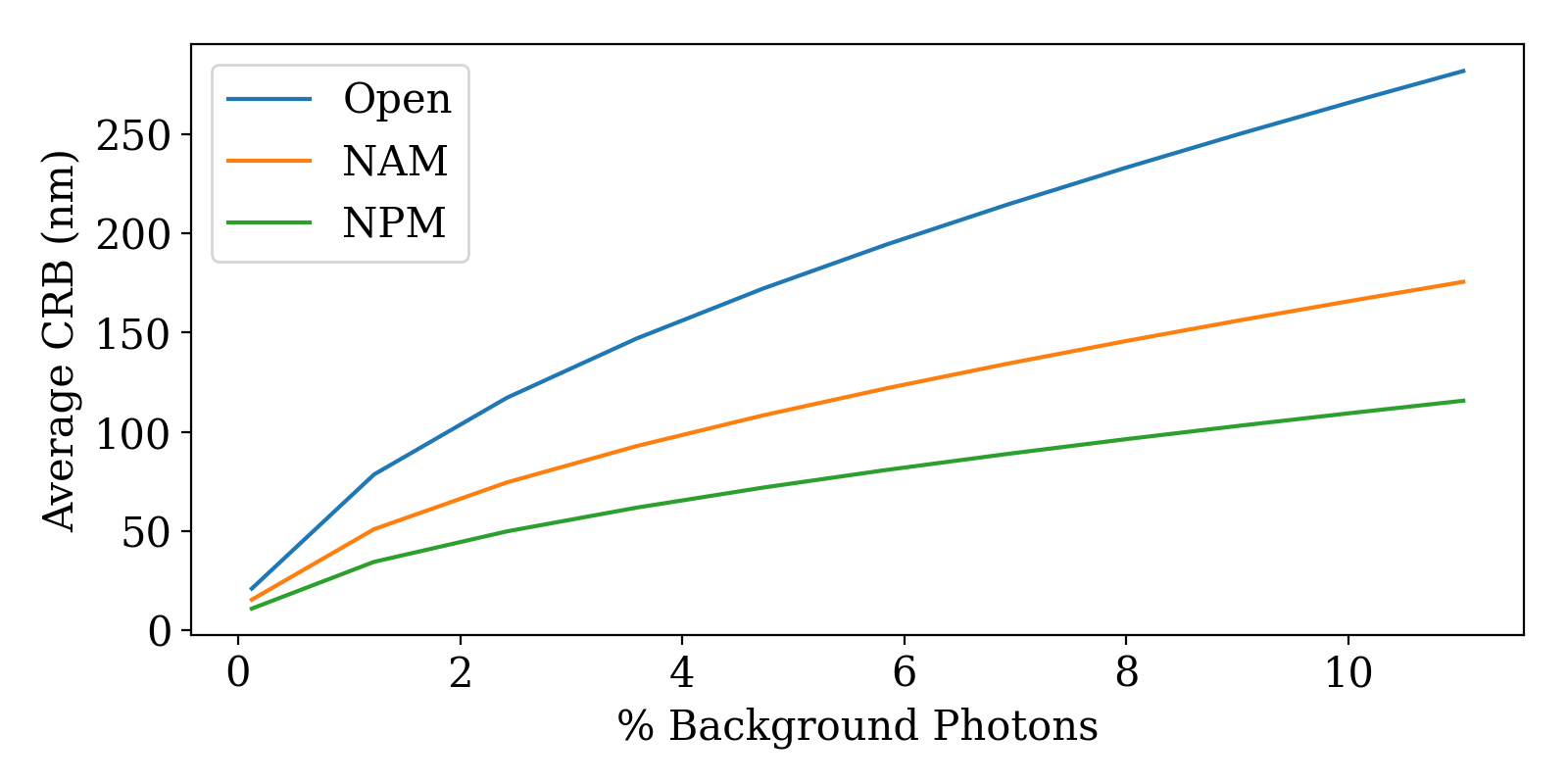}
    \caption{ \textbf{Background photon effect on CRB.} As the percentage of photons hitting the sensor due to background noise increases, CRB also increases. The impact is minimal in our method. }
    \label{fig:misc2}
    \vspace{-10pt}
\end{figure}

\section{Limitations}

While we were successful in designing optics to improve performance on 3D tracking with event cameras, our method carries some limitations. First, although our binned event frames can be obtained at kHz refresh rates, they do not take full advantage of the asynchronous nature of event cameras. Second, our bounds are for an idealized event camera model with no read-noise. It would be impossible to outperform these bounds, but there might exist a tighter bound that accounts for these hardware imperfections. Lastly, we only consider single-emitter images. With multiple point sources, the resolving accuracy between single points may be more limited.

\section{Conclusion}

This work introduces PSF-engineering to neuromorphic event-based sensors. We first derive information theoretical limits on 3D point localization and tracking. We demonstrate that existing amplitude and phase mask designs are suboptimal for tracking moving emitters and design new optical elements for this task. Additionally, to overcome the non-convexity of this optimization problem, we introduce a novel implicit neural representation for optical components. Finally, we validate the effectiveness of our designs in simulation and compare against state-of-the-art mask designs. Our work unlocks not only highly performant optics for event cameras but also the ability to design highly expressive elements for other sensors. 

\section*{Acknowledgements}
\vspace{-5pt}
\newtext{
This work was supported in part by the Joint Directed Energy Transition Office, AFOSR Young Investigator Program award no. FA9550-22-1-0208, ONR award no. N00014-23-1-2752 and N00014-17-1-2622, Dolby Labs, SAAB, Inc, and National Science Foundation grants BCS 1824198 and CNS 1544787. The support of the Maryland Robotics Center under a postdoctoral fellowship to C.S., is also gratefully acknowledged.
}

{
    \small
    \bibliographystyle{ieeenat_fullname}
    \bibliography{main}

\begin{thebibliography}{72}
\providecommand{\natexlab}[1]{#1}
\providecommand{\url}[1]{\texttt{#1}}
\expandafter\ifx\csname urlstyle\endcsname\relax
  \providecommand{\doi}[1]{doi: #1}\else
  \providecommand{\doi}{doi: \begingroup \urlstyle{rm}\Url}\fi

\bibitem[Abdelfattah et~al.(2023)Abdelfattah, Zheng, Singh, Huang, Reep,
  Tsegaye, Tsang, Arthur, Rehorova, Olson, Shuai, Zhang, Fu, Milkie, Moya,
  Weber, Lemire, Baker, Falco, Zheng, Grimm, Yip, Walpita, Chase, Campagnola,
  Murphy, Wong, Forest, Mertz, Economo, Turner, Koyama, Lin, Betzig, Novak,
  Lavis, Svoboda, Korff, Chen, Schreiter, Hasseman, and
  Kolb]{Abdelfattah:2023:VoltageSignal}
Ahmed~S. Abdelfattah, Jihong Zheng, Amrita Singh, Yi-Chieh Huang, Daniel Reep,
  Getahun Tsegaye, Arthur Tsang, Benjamin~J. Arthur, Monika Rehorova, Carl
  V.~L. Olson, Yichun Shuai, Lixia Zhang, Tian-Ming Fu, Daniel~E. Milkie,
  Maria~V. Moya, Timothy~D. Weber, Andrew~L. Lemire, Christopher~A. Baker,
  Natalie Falco, Qinsi Zheng, Jonathan~B. Grimm, Mighten~C. Yip, Deepika
  Walpita, Martin Chase, Luke Campagnola, Gabe~J. Murphy, Allan~M. Wong,
  Craig~R. Forest, Jerome Mertz, Michael~N. Economo, Glenn~C. Turner, Minoru
  Koyama, Bei-Jung Lin, Eric Betzig, Ondrej Novak, Luke~D. Lavis, Karel
  Svoboda, Wyatt Korff, Tsai-Wen Chen, Eric~R. Schreiter, Jeremy~P. Hasseman,
  and Ilya Kolb.
\newblock Sensitivity optimization of a rhodopsin-based fluorescent voltage
  indicator.
\newblock \emph{Neuron}, 111\penalty0 (10):\penalty0 1547--1563.e9, 2023.

\bibitem[Alzugaray and Chli(2018)]{Alzugaray:2018:FSAE}
Ignacio Alzugaray and Margarita Chli.
\newblock Asynchronous corner detection and tracking for event cameras in real
  time.
\newblock \emph{IEEE Robotics and Automation Letters}, 3\penalty0 (4):\penalty0
  3177--3184, 2018.

\bibitem[Amir et~al.(2017)Amir, Taba, Berg, Melano, McKinstry, Di~Nolfo, Nayak,
  Andreopoulos, Garreau, Mendoza, Kusnitz, Debole, Esser, Delbruck, Flickner,
  and Modha]{Amir:2017:gesture}
Arnon Amir, Brian Taba, David Berg, Timothy Melano, Jeffrey McKinstry, Carmelo
  Di~Nolfo, Tapan Nayak, Alexander Andreopoulos, Guillaume Garreau, Marcela
  Mendoza, Jeff Kusnitz, Michael Debole, Steve Esser, Tobi Delbruck, Myron
  Flickner, and Dharmendra Modha.
\newblock A low power, fully event-based gesture recognition system.
\newblock In \emph{Proceedings of the IEEE Conference on Computer Vision and
  Pattern Recognition (CVPR)}, 2017.

\bibitem[Angelopoulos et~al.(2021)Angelopoulos, Martel, Kohli, Conradt, and
  Wetzstein]{Angelopoulos:2021:eyetracking}
Anastasios~N. Angelopoulos, Julien~N.P. Martel, Amit~P.S. Kohli, Jorg Conradt,
  and Gordon Wetzstein.
\newblock Event based, near-eye gaze tracking beyond 10,000 hz.
\newblock \emph{IEEE Transactions on Visualization and Computer Graphics (Proc.
  VR)}, 2021.

\bibitem[Baek et~al.(2021)Baek, Ikoma, Jeon, Li, Heidrich, Wetzstein, and
  Kim]{Baek:2021:HS}
Seung-Hwan Baek, Hayato Ikoma, Daniel~S. Jeon, Yuqi Li, Wolfgang Heidrich,
  Gordon Wetzstein, and Min~H. Kim.
\newblock Single-shot hyperspectral-depth imaging with learned diffractive
  optics.
\newblock In \emph{Proceedings of the IEEE/CVF International Conference on
  Computer Vision (ICCV)}, pages 2651--2660, 2021.

\bibitem[Baldwin et~al.(2019)Baldwin, Almatrafi, Kaufman, Asari, and
  Hirakawa]{Baldwin:2019:LogDiff}
{R. Wes} Baldwin, Mohammed Almatrafi, {Jason R.} Kaufman, Vijayan Asari, and
  Keigo Hirakawa.
\newblock Inceptive event time-surfaces for object classification using
  neuromorphic cameras.
\newblock In \emph{Image Analysis and Recognition - 16th International
  Conference, ICIAR 2019, Proceedings}, pages 395--403, Germany, 2019. Springer
  Verlag.

\bibitem[Betzig et~al.(2006)Betzig, Patterson, Sougrat, Lindwasser, Olenych,
  Bonifacino, Davidson, Lippincott-Schwartz, and Hess]{Betzig:2006:PALM}
Eric Betzig, George~H. Patterson, Rachid Sougrat, O.~Wolf Lindwasser, Scott
  Olenych, Juan~S. Bonifacino, Michael~W. Davidson, Jennifer
  Lippincott-Schwartz, and Harald~F. Hess.
\newblock Imaging intracellular fluorescent proteins at nanometer resolution.
\newblock \emph{Science}, 313\penalty0 (5793):\penalty0 1642--1645, 2006.

\bibitem[Bouchard et~al.(2009)Bouchard, Chen, Burgess, and
  Hillman]{Bouchard:2009:BloodFlow}
Matthew~B Bouchard, Brenda~R Chen, Sean~A Burgess, and Elizabeth M~C Hillman.
\newblock Ultra-fast multispectral optical imaging of cortical oxygenation,
  blood flow, and intracellular calcium dynamics.
\newblock \emph{Opt Express}, 17\penalty0 (18):\penalty0 15670--15678, 2009.

\bibitem[Cabriel et~al.(2023)Cabriel, Monfort, Specht, and
  Izeddin]{Cabriel:2023:SMLM}
Cl{\'e}ment Cabriel, Tual Monfort, Christian~G. Specht, and Ignacio Izeddin.
\newblock Event-based vision sensor for fast and dense single-molecule
  localization microscopy.
\newblock \emph{Nature Photonics}, 2023.

\bibitem[Chang and Wetzstein(2019)]{Chang:2019:DeepOptics3D}
Julie Chang and Gordon Wetzstein.
\newblock Deep optics for monocular depth estimation and 3d object detection.
\newblock In \emph{Proc. IEEE ICCV}, 2019.

\bibitem[Clevert et~al.(2016)Clevert, Unterthiner, and
  Hochreiter]{Clevert:2016:ELU}
Djork-Arné Clevert, Thomas Unterthiner, and Sepp Hochreiter.
\newblock Fast and accurate deep network learning by exponential linear units
  (elus), 2016.

\bibitem[Feng et~al.(2023)Feng, Guo, Xie, Boominathan, Sharma, Veeraraghavan,
  and Metzler]{NeuWS}
Brandon~Y. Feng, Haiyun Guo, Mingyang Xie, Vivek Boominathan, Manoj~K. Sharma,
  Ashok Veeraraghavan, and Christopher~A. Metzler.
\newblock Neuws: Neural wavefront shaping for guidestar-free imaging through
  static and dynamic scattering media.
\newblock \emph{Science Advances}, 9\penalty0 (26):\penalty0 eadg4671, 2023.

\bibitem[Foix et~al.(2011)Foix, Alenya, and Torras]{Foix:2011:tof}
Sergi Foix, Guillem Alenya, and Carme Torras.
\newblock Lock-in time-of-flight (tof) cameras: A survey.
\newblock \emph{IEEE Sensors Journal}, 11\penalty0 (9):\penalty0 1917--1926,
  2011.

\bibitem[Gallego et~al.(2020)Gallego, Delbr{\"u}ck, Orchard, Bartolozzi, Taba,
  Censi, Leutenegger, Davison, Conradt, Daniilidis, et~al.]{Gallego:2020}
Guillermo Gallego, Tobi Delbr{\"u}ck, Garrick Orchard, Chiara Bartolozzi, Brian
  Taba, Andrea Censi, Stefan Leutenegger, Andrew~J Davison, J{\"o}rg Conradt,
  Kostas Daniilidis, et~al.
\newblock Event-based vision: A survey.
\newblock \emph{IEEE transactions on pattern analysis and machine
  intelligence}, 44\penalty0 (1):\penalty0 154--180, 2020.

\bibitem[Gao et~al.(2014)Gao, Liang, Li, and Wang]{Gao:2014:ultrafast}
Liang Gao, Jinyang Liang, Chiye Li, and Lihong~V. Wang.
\newblock Single-shot compressed ultrafast photography at one hundred billion
  frames per second.
\newblock \emph{Nature}, 516\penalty0 (7529):\penalty0 74--77, 2014.

\bibitem[Geng(2011)]{Geng:2011:StructuredLight}
Jason Geng.
\newblock Structured-light 3d surface imaging: a tutorial.
\newblock \emph{Adv. Opt. Photon.}, 3\penalty0 (2):\penalty0 128--160, 2011.

\bibitem[Ghanekar et~al.(2022)Ghanekar, Saragadam, Mehra, Gustavsson,
  Sankaranarayanan, and Veeraraghavan]{Ghanekar:2022:ps2f}
Bhargav Ghanekar, Vishwanath Saragadam, Dushyant Mehra, Anna-Karin Gustavsson,
  Aswin~C. Sankaranarayanan, and Ashok Veeraraghavan.
\newblock Ps $^{2}$ f: Polarized spiral point spread function for single-shot
  3d sensing.
\newblock \emph{IEEE Transactions on Pattern Analysis and Machine
  Intelligence}, pages 1--12, 2022.

\bibitem[Goodman(2017)]{Goodman:2017}
Joseph~W. Goodman.
\newblock \emph{Introduction to fourier optics}.
\newblock Freeman, 2017.

\bibitem[Griffin(1992)]{Griffin:1992}
Tralissa~F Griffin.
\newblock Distribution of the ratio of two poisson random variables.
\newblock Master's thesis, Texas Tech University, 1992.

\bibitem[Guo et~al.(2023)Guo, Yang, Chang, Hu, Greene, Gabel, You, and
  Tian]{Guo:2023:EventLFM}
Ruipeng Guo, Qianwan Yang, Andrew~S. Chang, Guorong Hu, Joseph Greene,
  Christopher~V. Gabel, Sixian You, and Lei Tian.
\newblock Eventlfm: Event camera integrated fourier light field microscopy for
  ultrafast 3d imaging, 2023.

\bibitem[Hartley and Zisserman(2004)]{Hartley:2004:MV}
R.~I. Hartley and A. Zisserman.
\newblock \emph{Multiple View Geometry in Computer Vision}.
\newblock Cambridge University Press, ISBN: 0521540518, second edition, 2004.

\bibitem[Hinojosa et~al.(2021)Hinojosa, Niebles, and
  Arguello]{Hinojosa:2021:privacy}
Carlos Hinojosa, Juan~Carlos Niebles, and Henry Arguello.
\newblock Learning privacy-preserving optics for human pose estimation.
\newblock In \emph{Proceedings of the IEEE/CVF International Conference on
  Computer Vision (ICCV)}, pages 2573--2582, 2021.

\bibitem[Hu et~al.(2021)Hu, Liu, and Delbruck]{Hu:2021:v2e}
Y Hu, S~C Liu, and T Delbruck.
\newblock v2e: From video frames to realistic {DVS} events.
\newblock In \emph{2021 {IEEE/CVF} Conference on Computer Vision and Pattern
  Recognition Workshops ({CVPRW})}. IEEE, 2021.

\bibitem[Iaboni et~al.(2021)Iaboni, Patel, Lobo, Choi, and
  Abichandani]{Iaboni:2021:Robots}
Craig Iaboni, Himanshu Patel, Deepan Lobo, Ji-Won Choi, and Pramod Abichandani.
\newblock Event camera based real-time detection and tracking of indoor ground
  robots.
\newblock \emph{IEEE Access}, 9:\penalty0 166588--166602, 2021.

\bibitem[Ikoma et~al.(2021)Ikoma, Nguyen, Metzler, Peng, and
  Wetzstein]{Ikoma:2021:DepthFromDefocus}
Hayato Ikoma, Cindy~M. Nguyen, Christopher~A. Metzler, Yifan Peng, and Gordon
  Wetzstein.
\newblock Depth from defocus with learned optics for imaging and
  occlusion-aware depth estimation.
\newblock \emph{IEEE International Conference on Computational Photography
  (ICCP)}, 2021.

\bibitem[Javier Hidalgo-Carrio and Scaramuzza(2020)]{Hidalgo:2020:EventDepth}
Daniel~Gehrig Javier Hidalgo-Carrio and Davide Scaramuzza.
\newblock Learning monocular dense depth from events.
\newblock \emph{{IEEE} International Conference on 3D Vision.(3DV)}, 2020.

\bibitem[Jusuf and Lew(2022)]{Jusuf:2022:OPSF}
James~M. Jusuf and Matthew~D. Lew.
\newblock Towards optimal point spread function design for resolving closely
  spaced emitters in three dimensions.
\newblock \emph{Opt. Express}, 30\penalty0 (20):\penalty0 37154--37174, 2022.

\bibitem[Kay(1993)]{Kay:1993}
Steven~M. Kay.
\newblock \emph{Fundamentals of Statistical Signal Processing}.
\newblock Prentice-Hall, 1 edition, 1993.

\bibitem[Kingma and Ba(2015)]{KingBa15}
Diederik Kingma and Jimmy Ba.
\newblock Adam: A method for stochastic optimization.
\newblock In \emph{International Conference on Learning Representations
  (ICLR)}, San Diega, CA, USA, 2015.

\bibitem[Kinkhabwala et~al.(2009)Kinkhabwala, Yu, Fan, Avlasevich, M{\"u}llen,
  and Moerner]{Kinkhabwala:2009}
Anika Kinkhabwala, Zongfu Yu, Shanhui Fan, Yuri Avlasevich, Klaus M{\"u}llen,
  and W.~E. Moerner.
\newblock Large single-molecule fluorescence enhancements produced by a bowtie
  nanoantenna.
\newblock \emph{Nature Photonics}, 3\penalty0 (11):\penalty0 654--657, 2009.

\bibitem[Lagorce et~al.(2017)Lagorce, Orchard, Galluppi, Shi, and
  Benosman]{Lagorce:2017:HOTS}
Xavier Lagorce, Garrick Orchard, Francesco Galluppi, Bertram~E Shi, and Ryad~B
  Benosman.
\newblock Hots: A hierarchy of event-based time-surfaces for pattern
  recognition.
\newblock \emph{IEEE Trans Pattern Anal Mach Intell}, 39\penalty0 (7):\penalty0
  1346--1359, 2017.

\bibitem[Lee et~al.(2014)Lee, Delbruck, Pfeiffer, Park, Shin, Ryu, and
  Kang]{Lee:2014:gesture}
Jun~Haeng Lee, Tobi Delbruck, Michael Pfeiffer, Paul K~J Park, Chang-Woo Shin,
  Hyunsurk~Eric Ryu, and Byung~Chang Kang.
\newblock Real-time gesture interface based on event-driven processing from
  stereo silicon retinas.
\newblock \emph{IEEE Trans Neural Netw Learn Syst}, 25\penalty0 (12):\penalty0
  2250--2263, 2014.

\bibitem[Levin et~al.(2007)Levin, Fergus, Durand, and Freeman]{levin2007image}
Anat Levin, Rob Fergus, Fr{\'e}do Durand, and William~T Freeman.
\newblock Image and depth from a conventional camera with a coded aperture.
\newblock \emph{ACM transactions on graphics (TOG)}, 26\penalty0 (3):\penalty0
  70--es, 2007.

\bibitem[Li et~al.(2022)Li, Wang, Song, Zhang, Xiong, and Huang]{Li:2022:HS}
Lingen Li, Lizhi Wang, Weitao Song, Lei Zhang, Zhiwei Xiong, and Hua Huang.
\newblock Quantization-aware deep optics for diffractive snapshot hyperspectral
  imaging.
\newblock In \emph{2022 IEEE/CVF Conference on Computer Vision and Pattern
  Recognition (CVPR)}, pages 19748--19757, 2022.

\bibitem[Liu et~al.(2020)Liu, Kuo, Antipa, Yanny, and
  Waller]{LindaLiu:2020:Lensless}
Fanglin~Linda Liu, Grace Kuo, Nick Antipa, Kyrollos Yanny, and Laura Waller.
\newblock Fourier diffuserscope: single-shot 3d fourier light field microscopy
  with a diffuser.
\newblock \emph{Opt. Express}, 28\penalty0 (20):\penalty0 28969--28986, 2020.

\bibitem[Liu et~al.(2022)Liu, Li, Liu, Hao, and Peng]{Liu:2022}
Xin Liu, Linpei Li, Xu Liu, Xiang Hao, and Yifan Peng.
\newblock Investigating deep optics model representation in affecting resolved
  all-in-focus image quality and depth estimation fidelity.
\newblock \emph{Opt. Express}, 30\penalty0 (20):\penalty0 36973--36984, 2022.

\bibitem[Llavador et~al.(2016)Llavador, Sola-Pikabea, Saavedra, Javidi, and
  Mart\'{i}nez-Corral]{Llavador:2016:LFM}
A. Llavador, J. Sola-Pikabea, G. Saavedra, B. Javidi, and M.
  Mart\'{i}nez-Corral.
\newblock Resolution improvements in integral microscopy with fourier plane
  recording.
\newblock \emph{Opt. Express}, 24\penalty0 (18):\penalty0 20792--20798, 2016.

\bibitem[Ma et~al.(2021)Ma, Lee, Best-Popescu, and Gao]{Ma:2021:Flourescence}
Yayao Ma, Youngjae Lee, Catherine Best-Popescu, and Liang Gao.
\newblock High-speed compressed-sensing fluorescence lifetime imaging
  microscopy of live cells.
\newblock \emph{Proceedings of the National Academy of Sciences}, 118\penalty0
  (3):\penalty0 e2004176118, 2021.

\bibitem[Maynard et~al.(2021)Maynard, Rostaing, Schaefer, Gemin, Candat,
  Dumoulin, Villmann, Triller, and Specht]{Maynard:2021}
Stephanie~A Maynard, Philippe Rostaing, Natascha Schaefer, Olivier Gemin,
  Adrien Candat, Andréa Dumoulin, Carmen Villmann, Antoine Triller, and
  Christian~G Specht.
\newblock Identification of a stereotypic molecular arrangement of endogenous
  glycine receptors at spinal cord synapses.
\newblock \emph{eLife}, 10:\penalty0 e74441, 2021.

\bibitem[Metzler et~al.(2020)Metzler, Ikoma, Peng, and
  Wetzstein]{Metzler:2020:HDR}
Christopher~A. Metzler, Hayato Ikoma, Yifan Peng, and Gordon Wetzstein.
\newblock Deep optics for single-shot high-dynamic-range imaging.
\newblock In \emph{IEEE/CVF Conference on Computer Vision and Pattern
  Recognition (CVPR)}, 2020.

\bibitem[Meurer et~al.(2017)Meurer, Smith, Paprocki, \v{C}ert\'{i}k, Kirpichev,
  Rocklin, Kumar, Ivanov, Moore, Singh, Rathnayake, Vig, Granger, Muller,
  Bonazzi, Gupta, Vats, Johansson, Pedregosa, Curry, Terrel, Rou\v{c}ka, Saboo,
  Fernando, Kulal, Cimrman, and Scopatz]{SymPy}
Aaron Meurer, Christopher~P. Smith, Mateusz Paprocki, Ond\v{r}ej
  \v{C}ert\'{i}k, Sergey~B. Kirpichev, Matthew Rocklin, AMiT Kumar, Sergiu
  Ivanov, Jason~K. Moore, Sartaj Singh, Thilina Rathnayake, Sean Vig, Brian~E.
  Granger, Richard~P. Muller, Francesco Bonazzi, Harsh Gupta, Shivam Vats,
  Fredrik Johansson, Fabian Pedregosa, Matthew~J. Curry, Andy~R. Terrel,
  \v{S}t\v{e}p\'{a}n Rou\v{c}ka, Ashutosh Saboo, Isuru Fernando, Sumith Kulal,
  Robert Cimrman, and Anthony Scopatz.
\newblock Sympy: symbolic computing in python.
\newblock \emph{PeerJ Computer Science}, 3:\penalty0 e103, 2017.

\bibitem[Mostafavi et~al.(2021)Mostafavi, Yoon, and
  Choi]{Mostafavi:2021:EventRGBStereo}
Mohammad Mostafavi, Kuk-Jin Yoon, and Jonghyun Choi.
\newblock Event-intensity stereo: Estimating depth by the best of both worlds.
\newblock In \emph{Proceedings of the IEEE/CVF International Conference on
  Computer Vision (ICCV)}, pages 4258--4267, 2021.

\bibitem[Nair and Hinton(2010)]{Nair:2010:SoftPlus}
Vinod Nair and Geoffrey~E. Hinton.
\newblock Rectified linear units improve restricted boltzmann machines.
\newblock In \emph{Proceedings of the 27th International Conference on
  International Conference on Machine Learning}, page 807–814, Madison, WI,
  USA, 2010. Omnipress.

\bibitem[Nam et~al.(2022)Nam, Mostafavi, Yoon, and Choi]{Nam:2022:EventStereo}
Yeongwoo Nam, Mohammad Mostafavi, Kuk-Jin Yoon, and Jonghyun Choi.
\newblock Stereo depth from events cameras: Concentrate and focus on the
  future.
\newblock In \emph{Proceedings of the IEEE/CVF Conference on Computer Vision
  and Pattern Recognition (CVPR)}, pages 6114--6123, 2022.

\bibitem[Ober et~al.(2004)Ober, Ram, and Ward]{Ober:2004}
Raimund~J Ober, Sripad Ram, and E~Sally Ward.
\newblock Localization accuracy in single-molecule microscopy.
\newblock \emph{Biophysical journal}, 86\penalty0 (2):\penalty0 1185--1200,
  2004.

\bibitem[Pavani et~al.(2009)Pavani, Thompson, Biteen, Lord, Liu, Twieg,
  Piestun, and Moerner]{Pavani:2009:DHPSF}
Sri Rama~Prasanna Pavani, Michael~A. Thompson, Julie~S. Biteen, Samuel~J. Lord,
  Na Liu, Robert~J. Twieg, Rafael Piestun, and W.~E. Moerner.
\newblock Three-dimensional, single-molecule fluorescence imaging beyond the
  diffraction limit by using a double-helix point spread function.
\newblock \emph{Proceedings of the National Academy of Sciences}, 106\penalty0
  (9):\penalty0 2995--2999, 2009.

\bibitem[Ranftl et~al.(2022)Ranftl, Lasinger, Hafner, Schindler, and
  Koltun]{Ranftl:2022:MiDaS}
Ren\'{e} Ranftl, Katrin Lasinger, David Hafner, Konrad Schindler, and Vladlen
  Koltun.
\newblock Towards robust monocular depth estimation: Mixing datasets for
  zero-shot cross-dataset transfer.
\newblock \emph{IEEE Transactions on Pattern Analysis and Machine
  Intelligence}, 44\penalty0 (3), 2022.

\bibitem[Rodríguez-Gómez et~al.(2022)Rodríguez-Gómez, Tapia, Garcia, Dios,
  and Ollero]{Gomez:2022:Robots}
Juan~Pablo Rodríguez-Gómez, Raul Tapia, Maria del Mar~Guzmán Garcia, Jose
  Ramiro Martínez-de Dios, and Anibal Ollero.
\newblock Free as a bird: Event-based dynamic sense-and-avoid for ornithopter
  robot flight.
\newblock \emph{IEEE Robotics and Automation Letters}, 7\penalty0 (2):\penalty0
  5413--5420, 2022.

\bibitem[Rust et~al.(2006)Rust, Bates, and Zhuang]{Rust:2006:STORM}
Michael~J Rust, Mark Bates, and Xiaowei Zhuang.
\newblock Sub-diffraction-limit imaging by stochastic optical reconstruction
  microscopy (storm).
\newblock \emph{Nature Methods}, 3\penalty0 (10):\penalty0 793--796, 2006.

\bibitem[Shah et~al.(2023)Shah, Kulshrestha, and Metzler]{Shah:2023:TiDy}
Sachin Shah, Sakshum Kulshrestha, and Christopher~A. Metzler.
\newblock Tidy-psfs: Computational imaging with time-averaged dynamic
  point-spread-functions.
\newblock In \emph{Proceedings of the IEEE/CVF International Conference on
  Computer Vision (ICCV)}, pages 10657--10667, 2023.

\bibitem[Sharonov and Hochstrasser(2006)]{Sharonov:2006:DNA-PAINT}
Alexey Sharonov and Robin~M. Hochstrasser.
\newblock Wide-field subdiffraction imaging by accumulated binding of diffusing
  probes.
\newblock \emph{Proceedings of the National Academy of Sciences}, 103\penalty0
  (50):\penalty0 18911--18916, 2006.

\bibitem[Shechtman et~al.(2014)Shechtman, Sahl, Backer, and
  Moerner]{Shechtman:2014}
Yoav Shechtman, Steffen~J. Sahl, Adam~S. Backer, and W.~E. Moerner.
\newblock Optimal point spread function design for 3d imaging.
\newblock \emph{Phys. Rev. Lett.}, 113:\penalty0 133902, 2014.

\bibitem[Shi et~al.(2023)Shi, Peng, Qiu, Ju, Lo, and Lo]{Shi:2023:EVEN}
Peilun Shi, Jiachuan Peng, Jianing Qiu, Xinwei Ju, Frank Po~Wen Lo, and Benny
  Lo.
\newblock Even: An event-based framework for monocular depth estimation at
  adverse night conditions, 2023.

\bibitem[Sironi et~al.(2018)Sironi, Brambilla, Bourdis, Lagorce, and
  Benosman]{Sironi:2018:HATS}
Amos Sironi, Manuele Brambilla, Nicolas Bourdis, Xavier Lagorce, and Ryad
  Benosman.
\newblock Hats: Histograms of averaged time surfaces for robust event-based
  object classification.
\newblock In \emph{Proceedings of the IEEE Conference on Computer Vision and
  Pattern Recognition (CVPR)}, 2018.

\bibitem[Sitzmann et~al.(2018)Sitzmann, Diamond, Peng, Dun, Boyd, Heidrich,
  Heide, and Wetzstein]{Sitzmann:2018:EDOF}
Vincent Sitzmann, Steven Diamond, Yifan Peng, Xiong Dun, Stephen Boyd, Wolfgang
  Heidrich, Felix Heide, and Gordon Wetzstein.
\newblock End-to-end optimization of optics and image processing for achromatic
  extended depth of field and super-resolution imaging.
\newblock \emph{ACM Transactions on Graphics (TOG)}, 37\penalty0 (4):\penalty0
  114, 2018.

\bibitem[Sitzmann et~al.(2020)Sitzmann, Martel, Bergman, Lindell, and
  Wetzstein]{Sitzmann:2019:SIREN}
Vincent Sitzmann, Julien~N.P. Martel, Alexander~W. Bergman, David~B. Lindell,
  and Gordon Wetzstein.
\newblock Implicit neural representations with periodic activation functions.
\newblock In \emph{Proc. NeurIPS}, 2020.

\bibitem[Small and Stahlheber(2014)]{Small:2014:SPL}
Alex Small and Shane Stahlheber.
\newblock Fluorophore localization algorithms for super-resolution microscopy.
\newblock \emph{Nature Methods}, 11\penalty0 (3):\penalty0 267--279, 2014.

\bibitem[Snyder and Miller(1991)]{Snyder:1991}
Donald~L. Snyder and Michael~I. Miller.
\newblock \emph{Random Point Processes in time and space}.
\newblock Springer, 2 edition, 1991.

\bibitem[Spencer et~al.(2020)Spencer, Bowden, and Hadfield]{Spencer:2020:Depth}
Jaime Spencer, Richard Bowden, and Simon Hadfield.
\newblock Defeat-net: General monocular depth via simultaneous unsupervised
  representation learning.
\newblock In \emph{IEEE/CVF Conference on Computer Vision and Pattern
  Recognition (CVPR)}, 2020.

\bibitem[Tinch et~al.(2022)Tinch, Menon, Hirakawa, and
  McCloskey]{Tinch:2022:stars}
Luc Tinch, Nitesh Menon, Keigo Hirakawa, and Scott McCloskey.
\newblock Event-based detection, tracking, and recognition of unresolved moving
  objects.
\newblock \emph{Advanced Maui Optical and Space Surveillance Technologies
  (AMOS) Conference}, 2022.

\bibitem[Tomasi and Kanade(1992)]{Tomasi:1992}
Carlo Tomasi and Takeo Kanade.
\newblock Shape and motion from image streams under orthography: a
  factorization method.
\newblock \emph{International Journal of Computer Vision}, 9\penalty0
  (2):\penalty0 137--154, 1992.

\bibitem[Verdier et~al.(2022)Verdier, Laurent, Cass{\'e}, Vestergaard, Specht,
  and Masson]{Verdier:2022}
Hippolyte Verdier, Fran{\c c}ois Laurent, Alhassan Cass{\'e}, Christian~L.
  Vestergaard, Christian~G. Specht, and Jean-Baptiste Masson.
\newblock A maximum mean discrepancy approach reveals subtle changes in
  $\alpha$-synuclein dynamics.
\newblock \emph{bioRxiv}, 2022.

\bibitem[Wu et~al.(2020)Wu, Liang, Chen, Hsu, Chavarha, Evans, Shi, Lin, Tsia,
  and Ji]{Wu:2020:twophoton}
Jianglai Wu, Yajie Liang, Shuo Chen, Ching-Lung Hsu, Mariya Chavarha, Stephen~W
  Evans, Dongqing Shi, Michael~Z Lin, Kevin~K Tsia, and Na Ji.
\newblock Kilohertz two-photon fluorescence microscopy imaging of neural
  activity in vivo.
\newblock \emph{Nat Methods}, 17\penalty0 (3):\penalty0 287--290, 2020.

\bibitem[Wu et~al.(2019)Wu, Boominathan, Chen, Sankaranarayanan, and
  Veeraraghavan]{Wu:2019:PhaseCam}
Yicheng Wu, Vivek Boominathan, Huaijin Chen, Aswin Sankaranarayanan, and Ashok
  Veeraraghavan.
\newblock Phasecam3d — learning phase masks for passive single view depth
  estimation.
\newblock In \emph{2019 IEEE International Conference on Computational
  Photography (ICCP)}, pages 1--12, 2019.

\bibitem[Xiao et~al.(2023)Xiao, Giblin, Boas, and Mertz]{Xiao:2023:twophoton}
Sheng Xiao, John~T. Giblin, David~A. Boas, and Jerome Mertz.
\newblock High-throughput deep tissue two-photon microscopy at kilohertz frame
  rates.
\newblock \emph{Optica}, 10\penalty0 (6):\penalty0 763--769, 2023.

\bibitem[Yang and Hinner(2015)]{Yang:2015:Diffusion}
Nicole~J Yang and Marlon~J Hinner.
\newblock Getting across the cell membrane: an overview for small molecules,
  peptides, and proteins.
\newblock \emph{Methods Mol Biol}, 1266:\penalty0 29--53, 2015.

\bibitem[Yin et~al.(2023)Yin, Zhang, Chen, Cai, Yu, Wang, Chen, and
  Shen]{Yin:2023:Metric3D}
Wei Yin, Chi Zhang, Hao Chen, Zhipeng Cai, Gang Yu, Kaixuan Wang, Xiaozhi Chen,
  and Chunhua Shen.
\newblock Metric3d: Towards zero-shot metric 3d prediction from a single image.
\newblock In \emph{Proceedings of the IEEE/CVF International Conference on
  Computer Vision (ICCV)}, pages 9043--9053, 2023.

\bibitem[You et~al.(2021)You, Tsai, Chiu, and Li]{You:2021:Depth}
Zunzhi You, Yi-Hsuan Tsai, Wei-Chen Chiu, and Guanbin Li.
\newblock Towards interpretable deep networks for monocular depth estimation.
\newblock In \emph{Proceedings of the IEEE/CVF International Conference on
  Computer Vision (ICCV)}, pages 12879--12888, 2021.

\bibitem[Zhang et~al.(2022)Zhang, Tang, Yu, Lu, and Huang]{Zhang:2022:Spike}
Jiyuan Zhang, Lulu Tang, Zhaofei Yu, Jiwen Lu, and Tiejun Huang.
\newblock Spike transformer: Monocular depth estimation for spiking camera.
\newblock In \emph{Computer Vision -- ECCV 2022}, pages 34--52, Cham, 2022.
  Springer Nature Switzerland.

\bibitem[Zhou and Nayar(2009)]{nayar2009what}
Changyin Zhou and Shree Nayar.
\newblock What are good apertures for defocus deblurring?
\newblock In \emph{2009 IEEE International Conference on Computational
  Photography (ICCP)}, pages 1--8, 2009.

\bibitem[Zhou et~al.(2009)Zhou, Lin, and Nayar]{nayar2009coded}
Changyin Zhou, Stephen Lin, and Shree Nayar.
\newblock Coded aperture pairs for depth from defocus.
\newblock In \emph{2009 IEEE 12th International Conference on Computer Vision},
  pages 325--332, 2009.

\bibitem[Zhu et~al.(2019)Zhu, Yuan, Chaney, and
  Daniilidis]{Zhu:2019:EventDepth}
Alex~Zihao Zhu, Liangzhe Yuan, Kenneth Chaney, and Kostas Daniilidis.
\newblock Unsupervised event-based learning of optical flow, depth, and
  egomotion.
\newblock In \emph{Proceedings of the IEEE/CVF Conference on Computer Vision
  and Pattern Recognition (CVPR)}, 2019.

\end{thebibliography}
}

\clearpage
\setcounter{page}{1}
\setcounter{section}{0}
\setcounter{figure}{0}
\renewcommand\thesection{S\arabic{section}}
\maketitlesupplementary

\section{Hardware Prototype} \label{supp:hardware}

\newtext{
We performed a real-world experiment for tracking a point light source at meter scale using a binary amplitude mask and a Prophessee EVK3 event camera. Specifically, we fabricated the NAM mask at 20mm diameter scale on a Creality Ender 3 S1 Pro using 1.75mm PLA filament (see \autoref{fig:realworld}). Then, we captured an event dataset by moving a point source at discrete depth planes ranging between $75$cm and $125$cm with and without our coded aperture. For all measurements, the camera was focused at $100$cm. We binned events in $1$ms intervals to achieve an effective frame rate of $1000$ FPS and trained a CNN to estimate the event frame's depth. Results in \autoref{fig:realworld-tracking} demonstrate improved tracking performance compared to an open aperture, particularly at depths where the point source is defocused.
}

\begin{figure}[t]
    \centering
    \includegraphics[width=0.3\linewidth]{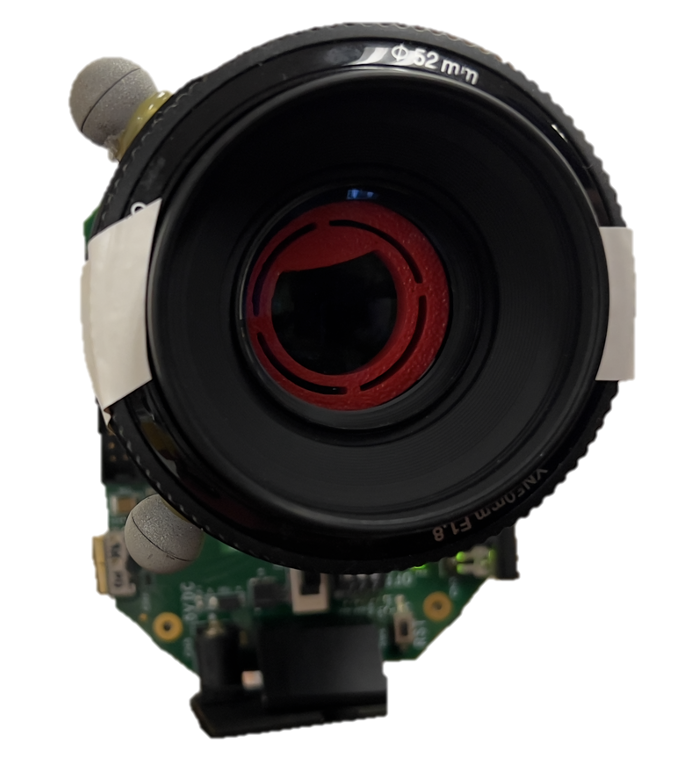}
     \includegraphics[width=0.9\linewidth]{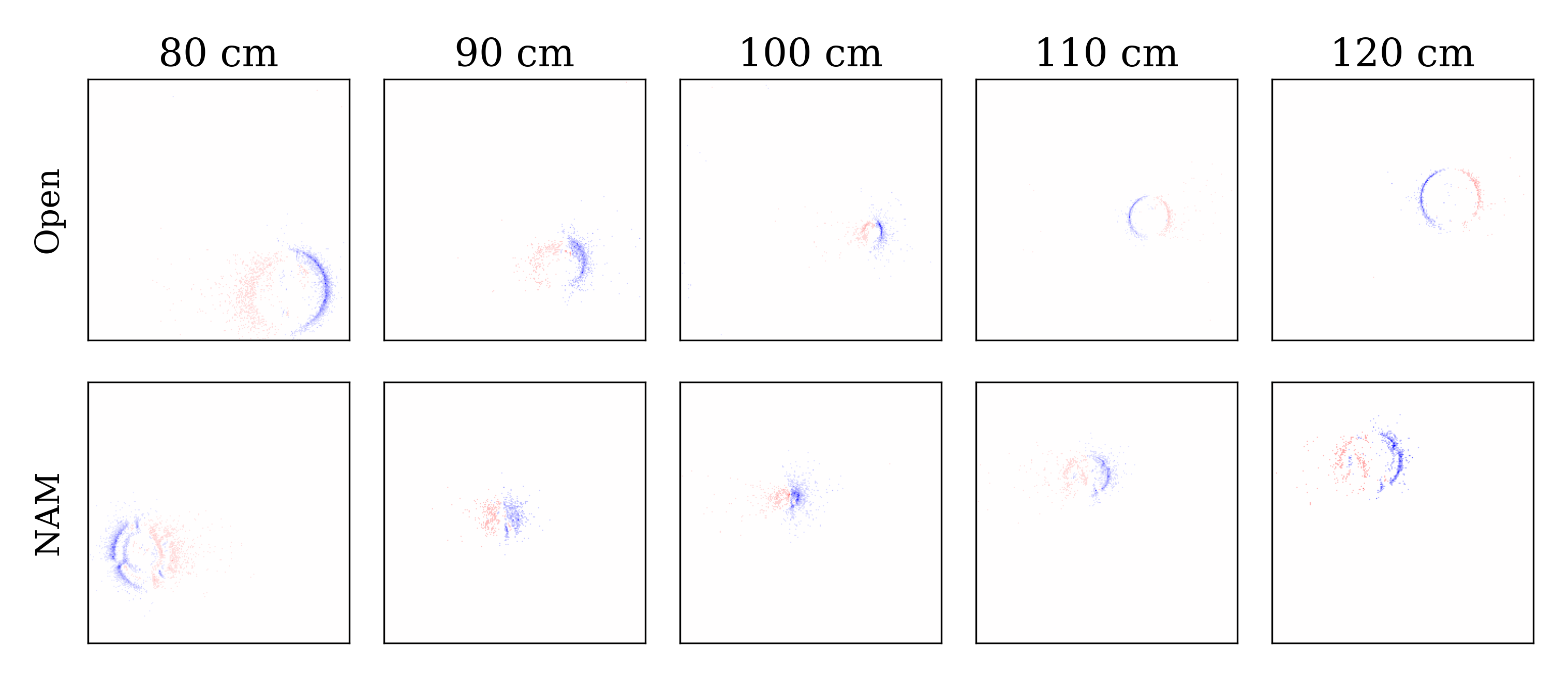}
    \caption{ \textbf{Prototype.} Top: The fabricated mask is placed at the aperture plane of an event camera with a $50$mm focal length lens. Bottom: Sample captured event frames for a point source.
    }
    \label{fig:realworld}
\end{figure}

\begin{figure}
    \centering
    \includegraphics[width=\linewidth]{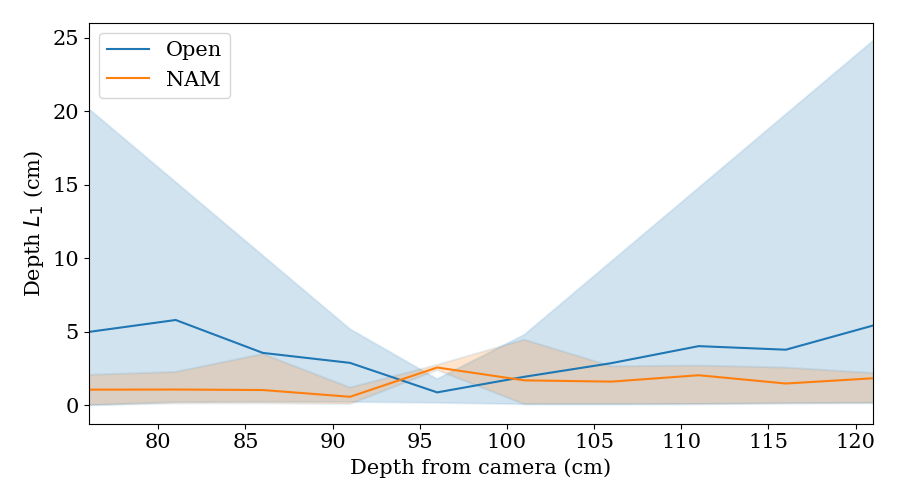}
    \caption{ \textbf{Real-world 3D tracking.} Comparison between NAM and Open apertures for depth estimation at 1000FPS. Error bars show the 90\% interquartile range.}
    \label{fig:realworld-tracking}
\end{figure}

\section{Accumulation Time}\label{supp:acc-time}
\newtext{
Cutting-edge event cameras offer $10$kHz fresh rates; even with $16$-frame accumulation, the camera effectively operates at $625$FPS --- much faster than conventional CMOS sensors.
We also retrained our CNN-based tracking algorithm on `pure' event frames with no accumulation. Overall performance degraded: NPM by $+45$\% RMSE and NAM by $+54$\% RMSE. Alternative architectures such as Spiking Neural Networks designed for sparse binary measurements may be better suited for processing `pure' events.
}

\section{The Effects of Particle Speed} \label{supp:part-speed}

We have shown CRB depends on particle speed; a natural question is does the optimal design change with respect to speed.
We optimize our neural phase mask using the CRB objective function with fixed particle speeds---$\{50, 100, 500, 1000\}$nm per time step.
Our learned designs are shown in \autoref{fig:evt-speed-learned}. When a particle moves quickly relative to the binned interval, the optimal design resembles the Fisher phase pattern found for traditional CMOS sensors. 

One can explain this collapse to the original Fisher mask design as follows. As a particle moves faster, the captured binned event frame looks more similar to the composition of a negative PSF at the start location and a positive PSF at the end location (\autoref{fig:evt-speed-effect}). This suggests that single-point event tracking mirrors two-point CMOS tracking.

\begin{figure}
    \centering
    \includegraphics[width=\linewidth]{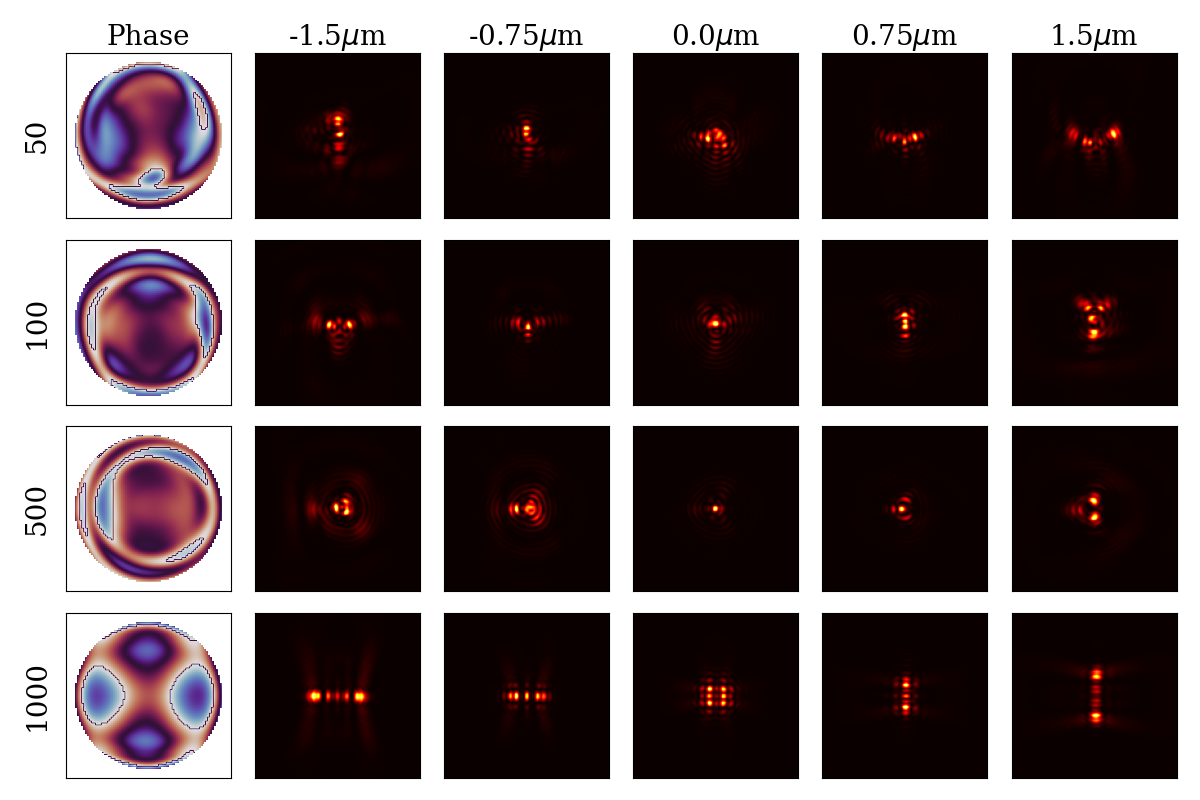}
    \caption{ \textbf{ Designed Phase Masks and corresponding PSFs for specific speeds. } Each row visualizes the neural phase mask designed for tracking particles moving at $N$ nanometers per time interval. Observe that the optimal design for `fast' moving particles is the Fisher design. }
    \label{fig:evt-speed-learned}
\end{figure}

\begin{figure*}
    \centering
    \includegraphics[width=\linewidth]{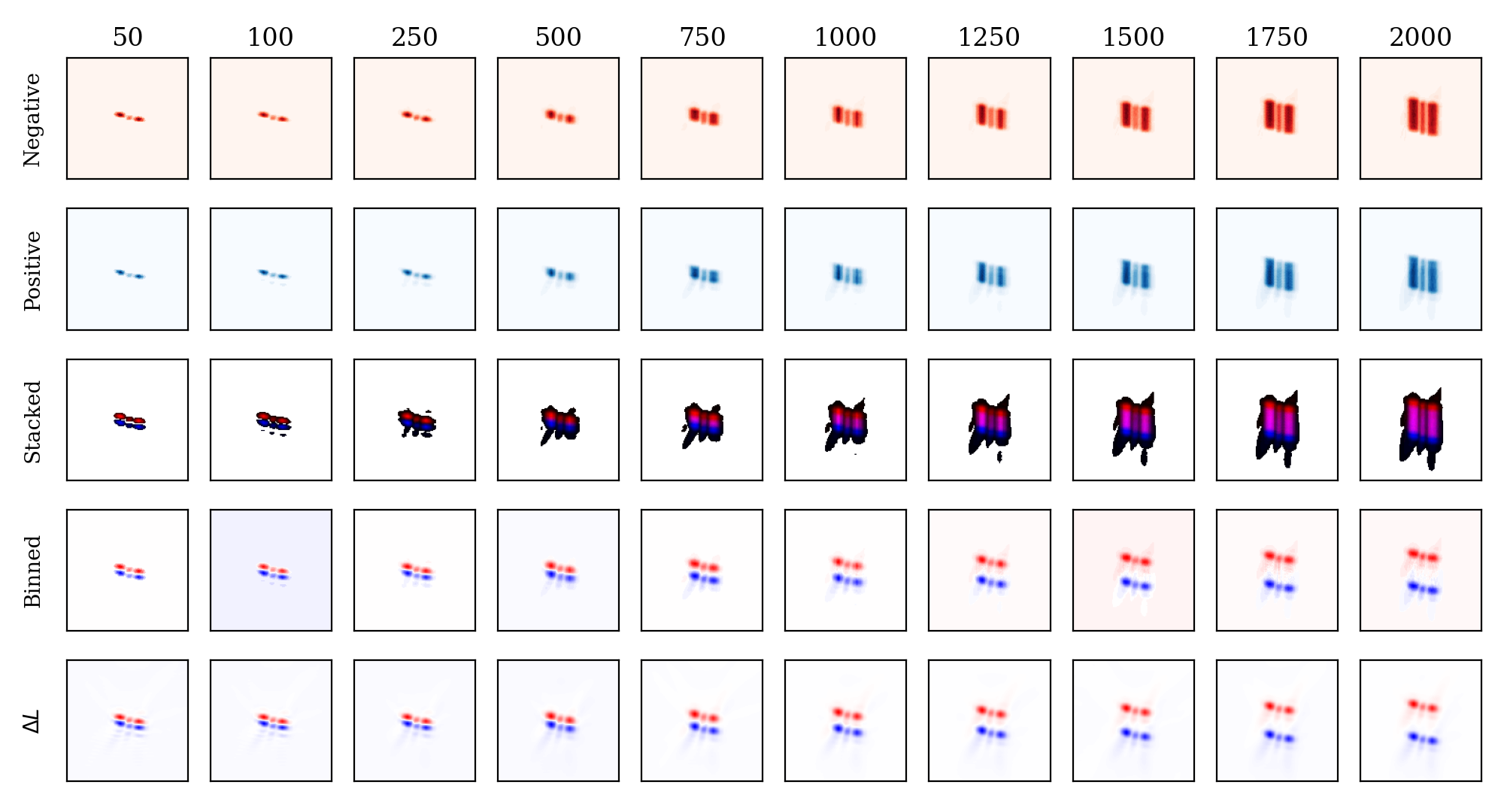}
    \caption{ \textbf{ Event camera measurements of a moving particle with the Fisher mask. } Motion is simulated over a fixed time interval with $100$ event samples. Observe a `fast' moving particle produces an event frame with two copies of a regular PSF: a negative copy at the start location, and a positive copy at the end location.
    \textit{Row 1}: negative event count over the time interval. \textit{Row 2}: positive event count over the time interval.
    \textit{Row 3}: the red channel visualizes negative events and the blue channel visualizes positive events. The pink regions represent where the events cancel in a binned measurement.
    \textit{Row 4}: binned event frame $pos - neg$.
    \textit{Row 5}: log-intensity difference $\Delta L$. }
    \label{fig:evt-speed-effect}
\end{figure*}

\section{Log-Intensity Difference Approximation}

In this section, we prove the log-intensity difference approximation we consider when deriving the Cram\'er Rao Bound is proportional to binned event frames.

Assume an idealized event camera model, where an event is triggered as soon as the log-intensity change between the reference and the current intensity equals some threshold, $\mathcal{T}$. Consider producing a binned event frame for a time interval $[t_\text{start}, t_\text{end}]$. For a single pixel, let the sequence of events over this interval occur at times $t_1, t_2, \ldots t_n$ and have polarities $p_1, p_2, \ldots, p_n\in\{-1, 1\}$. Let $f(t)$ be the log-intensity at time $t$ for the same pixel and be continuous over the interval.  
\begin{lemma}\label{thm:log-diff}
The log-intensity difference, $f(t_\text{end}) - f(t_\text{start})$, is proportional to the binned event pixel value, $\sum_{i=1}^n p_i$, with error $|\epsilon|<1$.
\begin{align}
    f(t_\text{end}) - f(t_\text{start}) \propto \epsilon + \sum_{i=1}^n p_i
\end{align}
\end{lemma}
\begin{proof}
By assumption, the magnitude of the change corresponding to each event is $\mathcal{T}$. Notice that $\mathcal{T}p_i$ is the log-intensity difference between the previous event time (the reference) and the current event time.
\begin{equation}
    \sum_{i=1}^n p_i = \frac{1}{\mathcal{T}} \sum_{i=1}^n f(t_i) - f(t_{i-1}) 
\end{equation}
The right-hand side is a telescoping sum,
\begin{equation}
    \sum_{i=1}^n f(t_i) - f(t_{i-1}) = f(t_n) - f(t_0).
\end{equation}
$t_0=t_\text{start}$ because the first event must occur $t_1 - t_0$ after the start of the interval. Then, the binned event frame is 
\begin{equation}
    \sum_{i=1}^n p_i = \frac{1}{\mathcal{T}}\pare{f(t_n) - f(t_\text{start})}.
\end{equation}
Finally, $|f(t_n) - f(t_\text{end})| = |\delta| < \mathcal{T}$ because if the quantity exceeded the threshold, an additional event would be triggered. Substitute $t_\text{end}$ for $t_n$.
\begin{align}
    \sum_{i=1}^n p_i &= \frac{1}{\mathcal{T}}\pare{f(t_\text{end}) - f(t_\text{start}) + \delta} \\
    &= \frac{1}{\mathcal{T}}\pare{ f(t_\text{end}) - f(t_\text{start}) } + \epsilon 
\end{align}
Thus, a binned event frame can be approximated as log-intensity difference divided by $\mathcal{T}$ with error $|\epsilon|<1$.
\end{proof}
As an event camera becomes more sensitive to change ($\mathcal{T}$ decreases), the approximation's percent error decreases because the magnitude of the binned event frame increases but the total absolute error is fixed at most $1$.

\end{document}